\documentclass[preprint,5p,times,twocolumn]{elsarticle}

\usepackage{amssymb}
\biboptions{numbers,sort&compress}

\usepackage{amsmath}
\usepackage{amsthm}

\newtheorem{proposition}{Proposition}
\newcommand{\tabincell}[2]{\begin{tabular}{@{}#1@{}}#2\end{tabular}}
\usepackage{threeparttable}
\usepackage{graphicx}  
\usepackage{float}  
\usepackage{subfigure}  
\usepackage{booktabs}
\usepackage{multirow}
\usepackage[colorlinks=True,
            linkcolor=black,
            anchorcolor=black,
            citecolor=black,
            urlcolor=black]{hyperref}
\usepackage{soul, color, xcolor}

\soulregister{\cite}7 
\soulregister{\citep}7 
\soulregister{\citet}7 
\soulregister{\ref}7 
\soulregister{\pageref}7 

\begin{document}

\begin{frontmatter}



\title{Improving Generalization of Neural Vehicle Routing Problem Solvers Through the Lens of Model Architecture} 
\author[1]{Yubin~Xiao}
\ead{xiaoyb21@mails.jlu.edu.cn}
\author[2,3]{Di~Wang}
\ead{wangdi@ntu.edu.sg}
\author[1]{Xuan~Wu\corref{cor1}}
\ead{wuuu22@mails.jlu.edu.cn}
\author[1]{Yuesong~Wu}
\ead{wuys23@mails.jlu.edu.cn}
\author[4]{Boyang~Li}
\ead{boyang.li@ntu.edu.sg}
\author[1]{Wei~Du}
\ead{weidu@jlu.edu.cn}
\author[1]{Liupu~Wang}
\ead{wanglpu@jlu.edu.cn}
\author[1]{You Zhou\corref{cor1}}
\ead{zyou@jlu.edu.cn}
\cortext[cor1]{Corresponding authors}

\affiliation[1]{organization={Key Laboratory of Symbolic Computation and Knowledge Engineering of Ministry of Education, College of Computer Science and Technology, Jilin University},
            city={Changchun},
            postcode={130012},
            country={China}}

\affiliation[2]{organization={Joint NTU-UBC Research Centre of Excellence in Active Living for the Elderly, Nanyang Technological University},
            postcode={639956},
            country={Singapore}}
\affiliation[3]{organization={WeBank-NTU Joint Research Institute on Fintech, Nanyang Technological University},
            postcode={639956},
            country={Singapore}}
\affiliation[4]{organization={College of Computing and Data Science, Nanyang Technological University},
            postcode={639956},
            country={Singapore}}




\begin{abstract}
Neural models produce promising results when solving Vehicle Routing Problems (VRPs), but may often fall short in generalization. Recent attempts to enhance model generalization often incur unnecessarily large training cost or cannot be directly applied to other models solving different VRP variants. To address these issues, we take a novel perspective on model architecture in this study. Specifically, we propose a plug-and-play Entropy-based Scaling Factor~(ESF) and a Distribution-Specific~(DS) decoder to enhance the size and distribution generalization, respectively. ESF adjusts the attention weight pattern of the model towards familiar ones discovered during training when solving VRPs of varying sizes. The DS decoder explicitly models VRPs of multiple training distribution patterns through multiple auxiliary light decoders, expanding the model representation space to encompass a broader range of distributional scenarios. We conduct extensive experiments on both synthetic and widely recognized real-world benchmarking datasets and compare the performance with seven baseline models. The results demonstrate the effectiveness of using ESF and DS decoder to obtain a more generalizable model and showcase their applicability to solve different VRP variants, i.e., travelling salesman problem and capacitated VRP. Notably, our proposed generic components require minimal computational resources, and can be effortlessly integrated into conventional generalization strategies to further elevate model generalization.
\end{abstract}




\end{frontmatter}



\section{Introduction}
The Vehicle Routing Problem (VRP) is a well-known Combinatorial Optimization (CO) problem with diverse real-world applications, e.g., transportation and logistics \cite{Kim2015}. Extensive research has been conducted on VRP over the years, led to the development of various solvers employing exact, approximate and heuristic algorithms \cite{Applegate2007, Helsgaun2017}. Nonetheless, existing solvers often suffer from the scalability issue \cite{Zhou2023} and require substantial manual rules and domain expertise \cite{Li2021}. As a promising alternative, Neural Network~(NN)-based models have been increasingly employed for solving VRPs in recent years, albeit with elusive theoretical guarantees. These models typically leverage learned NNs to acquire heuristics for constructing solutions or enhancing the quality of initial solutions through supervised learning or Reinforcement Learning (RL). By exploiting the underlying patterns of VRP training instances, neural VRP solvers achieve competitive or even superior solution quality compared to the conventional VRP solvers, accompanied by a noteworthy reduction in inference time. This work primarily focuses on the construction-based neural VRP solvers due to their faster inference speed \cite{Kool2019}. 

However, neural VRP solvers are typically trained and tested on instances with the same size and distribution, limiting their ability to generalize in unseen (especially in real world) scenarios that may deviate from the training set. For instance, the well-established POMO model \cite{Kwon2020} demonstrates excellent performance when trained and tested on VRP instances with a fixed size of nodes, where the node coordinates of training and test instances are both sampled from a uniform distribution. Nonetheless, the performance of the pretrained POMO model tends to decline when applied to unseen scenarios of different sizes (see Table~\ref{tab_log1}) and distribution patterns (see Table~\ref{tab_ds}). This drawback of low-level generalization poses a significant obstacle towards deploying these models, because real-world VRPs often involve tasks with varying sizes and unknown distribution. Consequently, enhancing the generalization of neural VRP solvers presents a practical yet challenging issue.

Several preliminary attempts have been made to enhance the generalization capability of neural VRP solvers. These attempts can be broadly categorized into the following two groups: size generalization and distribution generalization. Despite having promising outcomes, these approaches generally encounter the following three limitations: 1)~Size generalization methods predominantly focus on up-scaling, often neglecting the equally pertinent down-scaling generalization. This oversight is significant because practical VRPs may exhibit arbitrary sizes, necessitating comprehensive size generalization methods. 2)~Distribution generalization methods typically integrate various generalization algorithms, such as knowledge distillation \cite{Bi2022}, into the framework of model learning algorithms (e.g., RL). This integration amplifies the already considerable cost for model training. Furthermore, these methods may encounter three inherent training challenges, regardless of their specific training methodologies (see details in Section~\ref{dis_section}), often leading to under-performance. 3)~Both types of methods are often intricately designed, e.g., with numerous hyperparameters to be pre-set in the meta-learning process \cite{Zhou2023} and extended diffusion models specifically designed  for solving Travelling Salesman Problems (TSPs, a VRP variant) \cite{Sun2023}. The intricacy of these methods may constrain their applicability to a specific model \cite{Pan2023} or another VRP variant \cite{Sun2023}, impeding the potential adoption of diverse models to achieve further performance gain or solve other VRP variants.

To address the aforementioned limitations of the conventional generalization methods designed for neural VRP solvers, this study explores model generalization from a novel perspective---the model architecture. We aim to enhance model generalization by imposing lightweight model architecture improvement methods. This type of methods differs from existing generalization methods by excluding elaborately designed modules and training algorithms, thus making it potentially applicable to various models or VRP variants. Moreover, due to the unique perspective on enhancing model generalization, this type of methods can be integrated into existing generalization methods to achieve further performance elevation. Specifically, we propose two generic components based on the lens of model architecture to enhance generalization across varying sizes and distribution, respectively. For size generalization, we introduce an Entropy-based Scaling Factor (ESF) into the model's attention module, dynamically adjusting the attention weight pattern to align with patterns discovered during training. ESF accommodates both up-scaling and down-scaling generalization, and its plug-and-play nature allows effortless application during both testing and training. For distribution generalization, we leverage the nature of heavy encoder and light decoder architecture commonly used in neural VRP solvers \cite{Kwon2020}, and propose a Distribution-Specific~(DS) decoder-based method to explicitly model VRPs of different distribution patterns by using multiple light decoders. Compared to the conventional methods, our DS decoder-based one relies solely on the model learning algorithm, with no requirement of additional computation. Although both our proposed components have a straightforward design, to our great surprise and best of knowledge, such generic components have not been proposed to enhance the generalization of neural VRP solvers in the literature. 

Extensive experimental results, involving thirty eight synthetic and two real-world benchmarking (i.e., TSPLIB \cite{Reinelt1991} and CVRPLIB \cite{UCHOA2017}) datasets and seven baseline models \cite{Kool2019, Kwon2020, Xin2021, Bi2022, Zhou2023, Luo2023, Drakulic2023}, demonstrate the effectiveness of the proposed two components in enhancing model generalization. Additionally, by integrating the proposed components, existing generalization methods \cite{Bi2022, Zhou2023, Luo2023, Drakulic2023} for neural VRP solvers exhibit further improved generalization performance. More importantly, both components offer easy implementation and demand minimal computing resources (see Section~\ref{method}), which holds promising significance towards deployment in real-world applications and adoption by the neural CO community. 

The key contributions of this work are as follows.

{\romannumeral 1}) We propose a generic plug-and-play scaling factor applicable to neural VRP solvers in both testing and training phases, enhancing the size generalization performance with minimal computational overhead.

{\romannumeral 2}) We devise a generic distribution-specific decoder-based method to employ multiple light decoders for representation of different distribution patterns, thereby enhancing the model's ability to generalize across varying distribution.

{\romannumeral 3}) By adopting the proposed two components, we enhance the generalization of neural VRP solvers to avoid the problems of  both the intricacy of implementation and the substantial computational resources required by conventional methods.

{\romannumeral 4}) We show the effectiveness of the proposed two components in enhancing model generalization by conducting extensive experiments and ablation studies involving seven baseline models on forty datasets.

\section{Related Work}
In this section, we review the NN-based methods for solving VRPs, and then introduce several recent endeavors aimed at enhancing model generalization.

\subsection{Neural Network-Based VRP Solvers} 
NN-based methods have demonstrated promising results in solving VRPs \cite{Zhu2023, YANG2023, Xin2021a, Xiao2020, Wu2024, Zhao2024, Zhou2024, Fang2024} and can be broadly classified into the following two categories: 1)~Neural construction methods. These methods produce VRP solutions either incrementally, as in autoregressive models \cite{Vinyals2015, Jin2023}, or in a one-shot manner, as in non-autoregressive models \cite{Xiao2023, Min2023}.  For instance, Kool et al. \cite{Kool2019} proposed the well-known Transformer-based model (AM) based on the attention mechanism for solving VRPs. Subsequently, numerous studies extended AM and achieved better solution quality \cite{Kim2022, Kwon2021, Felix2023}, with POMO \cite{Kwon2020} emerged as the most representative model. Additionally, various post-processing methods have been proposed to further improve the performance of neural construction methods, such as EAS \cite{Hottung2022}, SGBS \cite{Choo2022}, and certain search methods \cite{Hottung2024, Garmendia2024}. Recently, Xiao et al. \cite{Xiao2024} introduced the GNARKD method, which distills a model that incrementally produces solutions into a model producing solutions in a one-shot manner. This approach notably decreases inference time while maintaining competitive solution quality, offering a novel perspective to the neural CO community. 2)~Neural improvement methods. These methods commence with initial solutions and employ specific deep learning techniques (e.g., a pretrained NN) to guide or assist heuristics to iteratively improve the solution. In line with this research, local search \cite{Xin2021a, Hudson2022, Ma2023} and evolutionary computation \cite{Ye2023, Kim2024} algorithms are often utilized. Our primary focus in this work is on the neural construction methods, because they often demonstrate on-par solution quality to neural improvement methods but with significantly shorter inference time. This aligns with the motivation behind employing NN-based models for solving VRPs, aiming to replace intensive computations with rapid approximations \cite{Bengio2021}. 

\subsection{Generalization of Neural VRP Solvers} Prior studies primarily focus on training and testing neural VRP solvers on instances of the same size and distribution, exhibiting limited generalization in unseen scenarios \cite{Joshi2022}. Several recent studies have been proposed to enhance the size and distribution generalization of these models, respectively. Size generalization methods aim to generalize the learned model to instances of varying sizes by incorporating a specifically designed size-dependent module \cite{Ahmad2023, Wang2024} or implementing the divide-and-conquer strategy \cite{Fu2021, Li2021a, Hou2023, Pan2023, Ye2024, Yu2024} to scale up to larger instances. Despite the current research's notable emphasis on scaling-up generalization, it is crucial not to overlook the prevalence of practical VRP instances with arbitrary sizes. Therefore, scaling-down generalization must also be considered---an aspect frequently neglected by existing size generalization methods. Distribution generalization methods aim to generalize models learned from VRPs of multiple predefined distribution patterns to instances with unseen distribution by adopting various generalization algorithms, such as curriculum learning \cite{Zhang2022}, adversarial training \cite{Geisler2022} and meta-learning \cite{Zhou2023}. In addition, researchers have also explored distribution generalization challenges in other domains beyond the context of VRPs \cite{Zhai2023, Moradi2024, Zhai2023a}. However, incorporating additional intricate generalization algorithms may demand resources that are often deemed unnecessarily large, because the computational resources used for training a neural VRP solver are already substantial. For instance, the implementation of AMDKD \cite{Bi2022} requires three additional pretrained teacher models. Furthermore, both size and distribution generalization methods are often intricately designed. For instance, the diffusion model DIFUSCO \cite{Sun2023} is specifically designed to solve TSPs (and maximal independent set, a CO problem) and cannot be applied to solve other VRP variants because the diffusion process in the model is not applicable to incorporate certain VRP constraints (e.g.,  the capacity constrain of CVRP). This complexity may limit their adaptability to a specific model or a VRP variant, hindering diverse model adoption for enhanced performance or application to other VRP variants.

To address the aforementioned challenging issues, we enhance the neural VRP solver's generalization capability by imposing lightweight model architecture improvement methods (i.e., ESF and the DS decoder), rather than introducing specifically designed neural modules and computationally expensive generalization algorithms.

Our ESF and DS decoder differ from prior studies in terms of methods \cite{Xin2021, Ahmad2023, Wang2024}. Unlike prior approaches that adopt variants of attention mechanism \cite{Ahmad2023} or enhance attention mechanisms in the last layer \cite{Wang2024}, ESF improves the vanilla attention across all layers and addresses down-scaling generalization. ESF’s broader applicability also allows potential integration with prior methods \cite{Wang2024}, a feature lacking in many existing approaches (e.g., \cite{Ahmad2023}). Our DS decoder differs from MDAM \cite{Xin2021} in three key areas: it prioritizes cross-distribution generalization over solution quality, produces a single solution instead of multiple for sampleable tasks, and distinguishes training distribution patterns rather than diversifying solutions with multiple decoders.


\section{Preliminaries}
This section presents the formulation of VRPs and introduces the model architecture commonly used in neural VRP solvers.

\subsection{VRP Setting}
We define a VRP-$n$ instance as a graph $\mathcal{G}$ comprising $n$ nodes, with the node coordinates $v$. The optimal solution of a VRP is the tour $\pi^*$ that visits all nodes with the minimum cost $\operatorname{c}(\pi^*)$, i.e., the tour's overall length is the shortest. Solving different VRP variants may be subject to various problem-specific constraints. This study specifically examines two prominent VRP variants, namely TSP and CVRP, due to their representativeness and widespread applications in various domains \cite{Kim2015}. In TSP, a feasible tour entails visiting each node exactly once. CVRP extends TSP by introducing an additional depot node, a capacity constraint of the delivery vehicle, and demand requests of each node that are smaller than the capacity constraint. A tour for CVRP consists of multiple sub-tours, each represents a vehicle starting from the depot, visiting a subset of nodes and subsequently returning to the depot. It is feasible if all nodes, except for the depot, are visited exactly once and the total demand in each sub-tour does not exceed the capacity constraint.

\subsection{Architecture of Neural VRP Solvers} \label{rel_arc}
Neural VRP solvers typically adopt an encoder-decoder framework. The encoder captures node features, while the decoder compute the probabilities of nodes to be visited based on the captured features. Nodes are selected sequentially until a tour $\pi$ is completed. For a given VRP instance $\mathcal{G}$, the probability of the tour is factorized using the chain rule as follows:
\begin{equation}\label{eq1}
p_\theta(\pi|\mathcal{G})=\prod_{t=1}^{T}p_\theta(\pi(t)|\pi(<t), \mathcal{G}),
\end{equation}
where $\pi(t)$ and $\pi(<t)$ denote the selected node and the current partial solution at time step $t$, respectively, $\theta$ denotes the learnable model parameters, and $T$ denotes the number of total steps. RL-based models typically define the negative cost $-\operatorname{c}(\pi)$ of a tour $\pi$ as the reward and then utilize the REINFORCE algorithm \cite{Williams1992} to estimate the gradient of the expected reward as follows:
\begin{equation}\label{eq2}
    \nabla_\theta\mathcal{L}(\theta|\mathcal{G})=\mathbb{E}_{P_{\theta}}[(\operatorname{c}(\pi)-\operatorname{c}(b(\mathcal{G}))\nabla_\theta \operatorname{log}\, p_{\theta}(\pi|\mathcal{G})],
\end{equation}
where $b(\mathcal{G})$ denotes the tour baseline for the VRP instance $\mathcal{G}$. Furthermore, the encoder and the decoder often comprise solely the feed-forward layers and the attention modules. The former involves the propagation of data, while the latter aids in mapping a query $Q\in\mathbb{R}^{n\times d_h}$, where $d_h$ denotes the hidden dimension, to an output using a set of key-value pairs $K\in\mathbb{R}^{n\times d_h}$ and $V\in\mathbb{R}^{n\times d_h}$ based on the Softmax function~$\operatorname{\sigma}$ as follows:
\begin{equation}\label{eq3}
    \operatorname{Attn}(Q, K, V)=\operatorname{\sigma}(QK^T)V,
\end{equation}
\begin{equation}\label{eq4}
   \operatorname{\sigma}(x_i) = \frac{e^{x_{i}}}{\sum_{j=1}^n e^{x_{j}}}, i\in\{1, 2, \dots, n\}.
\end{equation}

\section{Methodology}\label{method}
To improve the generalization of neural VRP solvers without relying on intricate generalization algorithms and specifically designed modules, we conduct generalization analysis from the lens of model architecture, and propose two generic components to enhance the size and distribution generalization, respectively.  

\subsection{Scaling Factor for Size Generalization}\label{size_section}
As introduced in Section~\ref{rel_arc}, neural VRP solvers commonly consist of only the feed-forward layers and the attention modules. The former exclusively involves the data $X\in\mathbb{R}^{d_h\times n}$ propagation and remains independent of the node number $n$, denoted as $WX\in\mathbb{R}^{d_h\times n}$, where $W\in\mathbb{R}^{d_h\times d_h}$ denotes the learnable parameters. Thus, for size generalization, our focus lies on the latter, which encompasses the computation of an attention weight matrix with a direct relationship to the size $n$, i.e., $\sigma(QK^T)\in\mathbb{R}^{n\times n}$. Notably, this perspective aligns with the conclusion drawn by \cite{Ahmad2023}, who proposed a sparse dynamic attention module. However, it is important to clarify that our objective is not to formulate a novel attention module but to improve the vanilla attention module by adjusting the attention weight matrix, because the latter has a significantly broader applicability with minimal computational overhead.

Inspired by NN-based models aiming to discern a fixed mapping pattern between inputs and outputs \cite{Zhong2021a, Joshi2022}, we attribute the model's insufficient size generalization performance (partially) to the changes in the attention weight pattern when solving VRPs of different sizes. Specifically, the feed-forward layer, which is computational independent of the size, allows a size-independent pattern in the element $x\in QK^T$ because matrices $Q$ and $K$ are computed from the feed-forward layers. Given the monotonically increasing nature of the function $e^{x}$, we deduce that the attention score element $x\in e^{QK^T}$ also adhere to a size-independent pattern. Therefore, regarding the Softmax function $\operatorname{\sigma}$ (see~(\ref{eq4})) in the attention module, the numerator $e^{x}$ follows a constant size-independent pattern, whereas the denominator's value $\sum_{j=1}^n e^{x_{j}}$ varies along with the size $n$ due to the accumulation operation $\sum_{j=1}^n$. Consequently, the resulting attention weight $\operatorname{\sigma}(QK^T)$ experiences dilution or concentration when the size increases or decreases. This leads to an unfamiliar attention weight pattern for the model when solving VRPs of different sizes. 

To alleviate this phenomenon, we further study the entropy value of the attention weight, which is a metric for measuring the informativeness of a pattern. In Proposition~\ref{pro}, we formally demonstrate that the lower bound of the attention weight entropy varies with the size $n$. Variations in the information quantity result in unfamiliar patterns that impact the generalization performance of the model.

\begin{proposition}\label{pro}
Let $\{x_1, x_2, \dots, x_n\}$ denotes a row in the  matrix $QK^T$. Then, the lower bound $\Omega$ of the entropy $\operatorname{\mathcal{H}}$ for the attention weights $\sigma(QK^T)$ varies based on the size $n$ as follows:
\begin{equation}
    \operatorname{\mathcal{H}}(\operatorname{\sigma}(x_i|1\leq i\leq n))\geq\Omega(\operatorname{ln}n).
\end{equation}
\end{proposition}

\begin{proof}
    According to the definition of entropy for discrete values, denoted as $\operatorname{\mathcal{H}}(p)=-\sum_i^n p_i \operatorname{ln}p_i$, we compute the attention weight entropy as follows:
    \begin{equation}
        \begin{aligned}
        \operatorname{\mathcal{H}}(\operatorname{\sigma}(x_i))&=-\sum_i^n \frac{e^{x_i}}{\sum_j^n e^{x_j}} \operatorname{ln}\frac{e^{x_i}}{\sum_j^n e^{x_j}}\\
           &=-\sum_i^n \frac{e^{x_i}}{\sum_j^n e^{x_j}}\left(x_i-\operatorname{ln}\sum_j^n e^{x_j}\right)\\
           &=-\sum_i^n\frac{e^{x_i}}{\sum_j^n e^{x_j}} x_i + 1\cdot \operatorname{ln}\sum_j^n e^{x_j}\\
           &=\operatorname{ln}\sum_j^n e^{x_j}-\mathbb{E}(x_i).
        \end{aligned}
    \end{equation}
Furthermore, we denote the minimum and maximum values of $x_i$ as $x^\textit{min}$ and $x^\textit{max}$, respectively, i.e., $x_i\in[x^\textit{min}, x^\textit{max}]$. Then, the attention weight entropy is deduced as follows:
    \begin{equation}
        \begin{aligned}
        \operatorname{\mathcal{H}}(\operatorname{\sigma}(x_i))&\geq \operatorname{ln}(ne^{x^\textit{min}})-\underset{i\in\{1,\dots, n\}}{\max}x_i\\
        &=\ln n+x^\textit{min}-x^\textit{max}\\
        &=\Omega(\operatorname{ln}n).
        \end{aligned}
    \end{equation}
Due to the adherence of the element $x_i$ to a size-independent pattern, the difference between the minimum value $x^\textit{min}$ and the maximum value $x^\textit{max}$, i.e., $x^\textit{max} - x^\textit{min}$ can be regarded as a constant independent of size. Therefore, the lower bound of entropy $\operatorname{\mathcal{H}}$ for the attention weight $\sigma(QK^T)$ is related to the size $n$,  specifically $\ln n$, illustrating the entropy value of the attention weight varies along with the size $n$. 
\end{proof}


Intuitively, the model's size generalization can be improved by adjusting the attention weight patterns, ensuring that the entropy of attention weight during testing is close to that discovered during training. However, it is worth noting that NN typically includes extensive nonlinear transformation operations, which makes it nearly impossible to precisely align the entropy of attention weights during both testing and training. Drawing from the logarithmic relationship observed between the attention weight entropy $\operatorname{\mathcal{H}}$ and the size $n$, specifically $\operatorname{\mathcal{H}}\geq\Omega(\operatorname{ln}n)$, we introduce a scaling factor to approximate the pattern of attention weight during both testing and training. Formally, assume that the model is trained on the VRP-$n_{\textit{tr}}$ instances of size $n_{\textit{tr}}$, we define the scaling factor as $\operatorname{log}_{n_{\textit{tr}}}n_{\textit{te}}$ when solving VRP-$n_{\textit{te}}$ instances of size $n_{\textit{te}}$, and incorporate it into the attention module as follows:
\begin{equation}\label{eq5}
    \operatorname{Attn}(Q, K, V)=\operatorname{\sigma}(QK^T \cdot\operatorname{log}_{n_{\textit{tr}}}n_{\textit{te}})V.
\end{equation}
\begin{figure}[!t]
\centering
\subfigure[Model trained at a fixed size]{\includegraphics[width=.495\columnwidth]{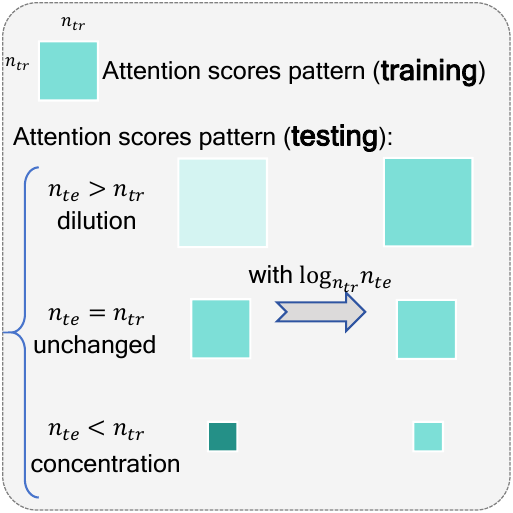}\label{sf1}}
\subfigure[Model trained at varying sizes]{\includegraphics[width=.495\columnwidth]{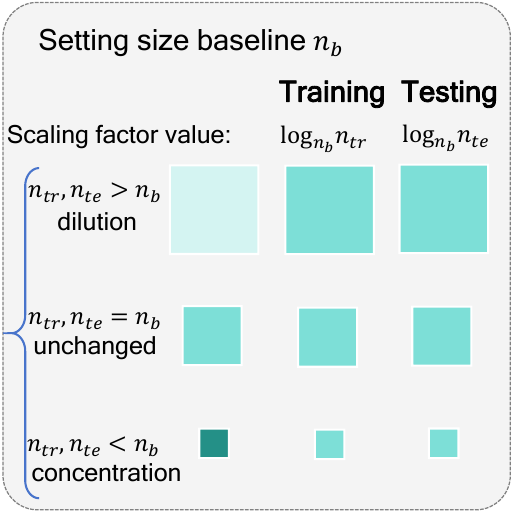}\label{sf2}} 
\caption{Illustration on how the entropy-based scaling factor regularizes attention weight patterns.}\label{sc}
\end{figure}
We depict the dynamics of this Entropy-based Scaling Factor (ESF) in Figure~\ref{sc}. Our proposed ESF has noteworthy traits as follows. 1) It does not compromise the model's performance when solving VRPs with the test size $n_{\textit{te}}$ equal to the training size $n_{\textit{tr}}$ because its value $\operatorname{log}_{n_{\textit{tr}}} n_{\textit{te}}$ is equal to 1, i.e., it falls back to the vanilla attention module $\operatorname{\sigma}(QK^T \cdot1)V$. 2)~It effectively mitigates the problems associated with attention weight dilution ($\operatorname{log}_{n_{\textit{tr}}}n_{\textit{te}}>1$) or concentration ($\operatorname{log}_{n_{\textit{tr}}}n_{\textit{te}}<1$) caused by the change in node size (refer to Figure~\ref{sf1}). 3) When models (e.g., \cite{Zhou2023}) are trained on VRPs of varying sizes, ESF is able to further enhance generalization performance by establishing a hyperparameter $n_{b}$ representing the size baseline. Specifically, when the training size $n_{\textit{tr}}$ is not fixed, we incorporate the scaling factor of $\operatorname{log}_{n_{b}}n_{\textit{tr}}$ into the attention module during training and use the attention module with the scaling factor of $\operatorname{log}_{n_{b}}n_{\textit{te}}$ during testing (refer to Figure~\ref{sf2}). Considering that most neural VRP solvers produce near-optimal solutions on VRP-50s (i.e., VRPs with 50 nodes), we set the hyperparameter $n_{b}$ to $50$ in all experiments in this work.


In summary, ESF is plug-and-play in nature for effortless implementation in practice. It only needs to be applied within each attention module of neural VRP solvers to achieve size generalization improvement.

\subsection{DS Decoder for Distribution Generalization}\label{dis_section}
Existing distribution generalization methods commonly train the neural VRP solvers with instances exhibiting multiple predefined distribution patterns, employing different intricate generalization algorithms \cite{Bi2022, Zhou2023}. However, such approaches pose three notable challenges, as illustrated in Figure~\ref{chal}: 1)~Struggling to reconcile conflicts among training instances of different distribution patterns, which is widely recognized as a challenging issue in multi-objective learning \cite{Chen2023}; 2)~Potentially leading to model degeneracy, where a model trained under multiple distribution patterns may exhibit suboptimal performance compared to the model exclusively trained under a single distribution pattern when solving VRPs from that specific distribution pattern (see Table~\ref{tab_ds}); and 3)~A model trained under multiple distribution patterns may exhibit a greater performance degradation than the model trained under a single distribution pattern when the distribution of the test cases is close to that single distribution pattern, primarily due to the joint distribution drift \cite{haug2021}. The second challenge arises when the test distribution pattern exactly matches one of the training distribution patterns. In contrast, the third challenge arises when the test distribution pattern only approximates a training distribution pattern. Compared to the second challenge, the third challenge is more realistic, because the test distribution patterns are often unknown. Furthermore, these challenges could be inherent to various generalization algorithms, regardless of their specific training methodologies. This is because the distribution space represented by a single output model struggles to adequately cover the space of the multiple training distribution patterns (see the context of Table~\ref{tab_ds}). Essentially, these generalization algorithms learn the intersection space of multiple training distribution patterns, aiming to find a model that performs reasonably well across all these distribution patterns without significant degradation. For instance, AMDKD \cite{Bi2022} adaptively chooses a specific distribution pattern that the model performs suboptimally at each epoch for learning to improve the overall performance across multiple training distribution patterns. On the contrary, we explicitly model each training distribution pattern using a model architecture improvement method. This approach aims to enable the model to learn both the intersection space across multiple training distribution patterns and the distinctions between them, thereby generalizing the distribution space represented by the model.
\begin{figure}[!t]
	\centering
\includegraphics[width=1\columnwidth]{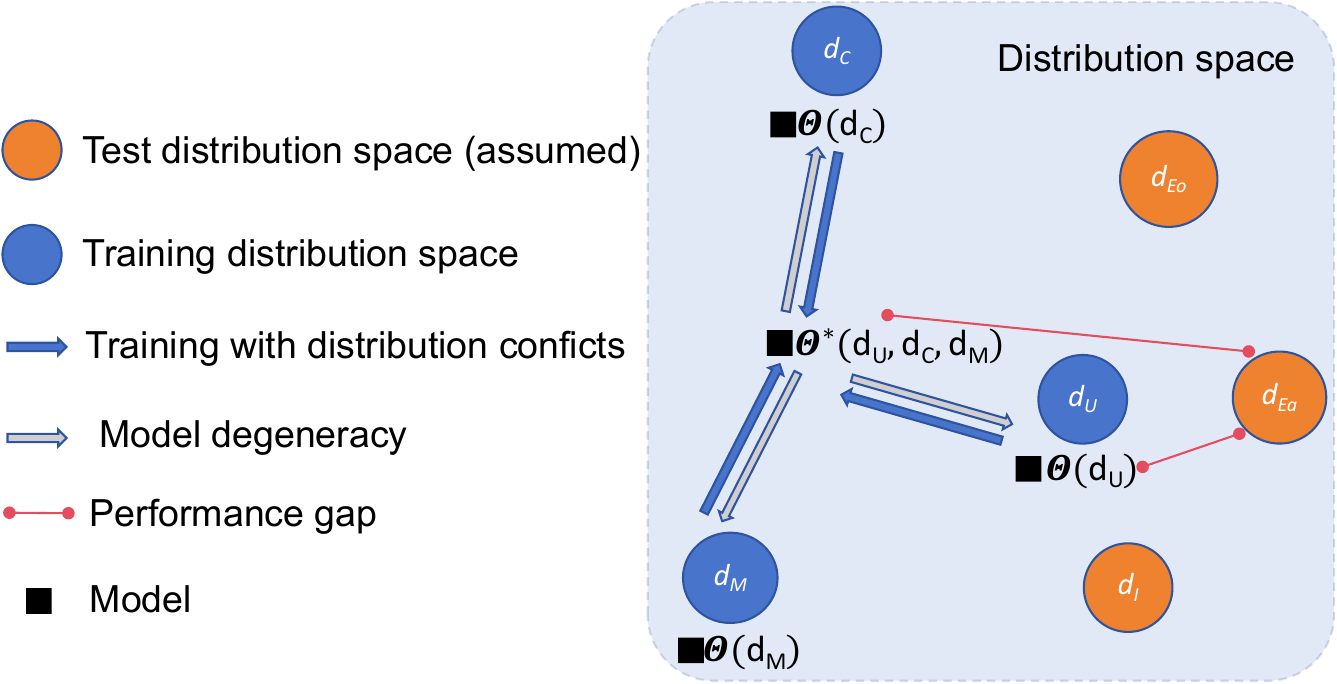}
	\caption{An illustrative example on training models with various distribution patterns. Here we denote the model parameter $\theta$ trained on the distribution $d_i$ as $\theta(d_i)$. The model trained on multi-distribution patterns (e.g., $\theta(d_U, d_C, d_M)$ may struggle to adequately cover the space of the multiple training distribution patterns ($d_U, d_C, d_M$).}\label{chal}
\end{figure}
\begin{figure*}[!t]
	\centering
\includegraphics[width=2.1\columnwidth]{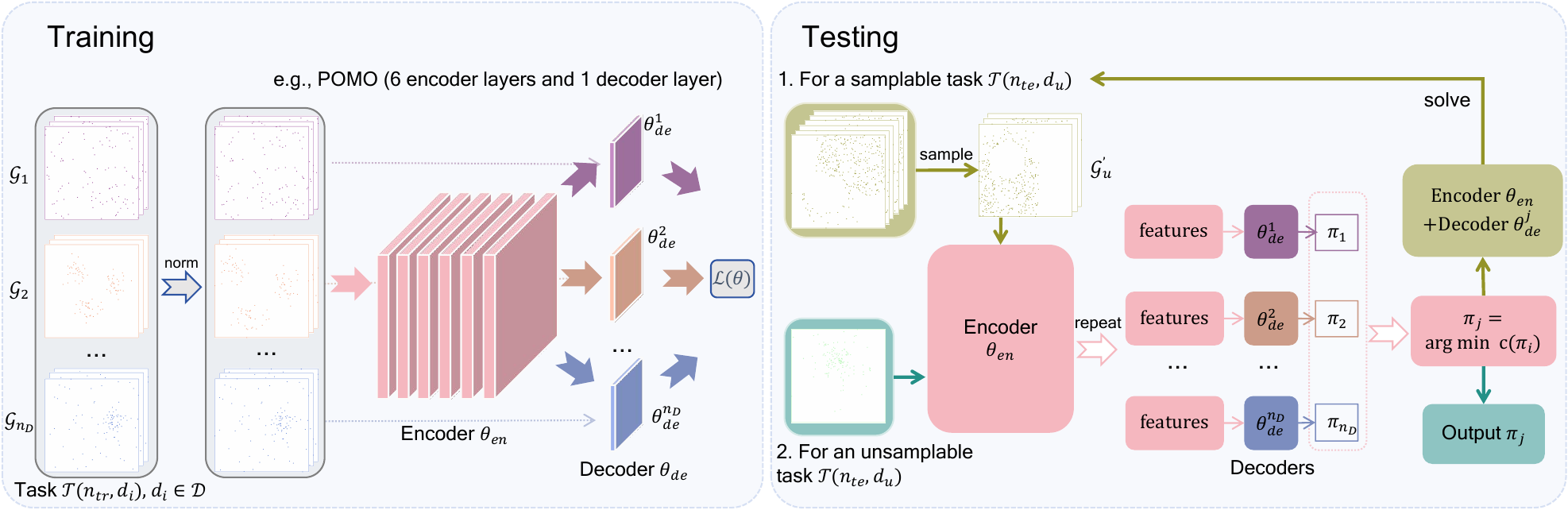}
	\caption{Architecture of the DS decoder, which employs multiple light decoders to model the representation of VRPs with multiple distribution patterns. }\label{de_dec}
\end{figure*}



Due to the success of the prevalent model architecture of heavy encoder and light decoder (e.g., POMO \cite{Kwon2020} employs six encoder layers and one decoder layer), we employ a shared heavy encoder to infer the fixed representation for VRPs across varying distribution patterns. Simultaneously, we employ multiple light decoders to model the representation of VRPs with multiple distribution patterns (see Section~\ref{rel_arc} for the relationship between the encoder and decoder of model architecture). The implementation of these Distribution-Specific (DS) decoders allows for explicitly incorporating the distinction among VRPs of different distribution patterns into the light decoder. By predefining multiple training distribution patterns for learning, as done in prior studies \cite{Zhou2023}, our DS decoder inherently enables the model to discern the optimal representation for each distribution pattern, incurring only marginal computational cost of auxiliary light decoders\footnote{A decoder typically includes only a feed-forward layer and an attention module \cite{Kool2019, Kwon2020}.}. We present the architecture of the proposed DS decoder in Figure~\ref{de_dec}. In the following paragraphs, we detail the training and testing process of applying the DS decoder. \\

\noindent\textbf{Training.} We define a graph $\mathcal{G}_i$ of a training instance sampled from a task $\mathcal{T}(n_{\textit{tr}}, d_i)$ with size $n_{\textit{tr}}$ and distribution $d_i\in \mathcal{D}$, where $\mathcal{D}$ denotes a set of distribution patterns with a total count of $n_{\mathcal{D}}$. Then, the model comprises one encoder and $n_{\mathcal{D}}$ decoders, where each decoder corresponds to a training distribution pattern. The detailed process of generating training instances of different distribution patterns follows \cite{Zhou2023} exactly. Considering that the representation of certain distribution patterns, such as the cluster distribution \cite{Bi2022}, is limited to a small segment of the overall graph, we perform normalization $\phi(\cdot)$ to the node coordinates $v$ of training instances as follows:
\begin{equation}\label{eq10}
    \phi(v) = \frac{v-\min_{v\in\mathcal{G}_i}(v)}{\max_{c\in \{x, y\}}(\max_{i, j\in\{1,\dots, n\}}(v^c_i-v^c_j))},
\end{equation}
where $v^x_i$ and $v^y_i$ denote the $x$ and $y$-coordinates of node $v_i$, respectively. Then, we formulate the training objective as follows:
\begin{equation}\label{eq6}
    \theta^*=\arg \min \limits_{\theta}\mathbb{E}_{\mathcal{G}_i\sim \mathcal{T}(n_{\textit{tr}}, d_i), d_i\in\mathcal{D}, \pi\sim  p_\theta(\pi|\mathcal{G}_i)}\operatorname{c}(\pi|\mathcal{G}_i).
\end{equation}
Furthermore, the gradient of (\ref{eq6}) can be estimated using (\ref{eq2}) as follows:
\begin{equation}\label{eq7}
\begin{aligned}
    \nabla_\theta\mathcal{L}(\theta)=&\frac{1}{n_{\mathcal{D}}} \sum_{i=1}^{n_{\mathcal{D}}}\mathbb{E}_{\mathcal{G}_i\sim \mathcal{T}(n_{\textit{tr}}, d_i), p_{\theta_{\textit{en}},\theta_{\textit{de}}^i}(\pi|\mathcal{G}_i)}[(\operatorname{c}(\pi)-\operatorname{c}(b(\mathcal{G}_i))\\&\cdot\nabla_{\theta_{\textit{en}},\theta_{\textit{de}}^i}  \operatorname{log}\, p_{\theta_{\textit{en}},\theta_{\textit{de}}^i}(\pi|\mathcal{G}_i)],
\end{aligned}
\end{equation}
where $\theta_{\textit{en}}$ and $\theta_{\textit{de}}^i$ denote the parameters of the encoder and the $i$th decoder, respectively. \\

\noindent\textbf{Testing.} For a samplable task $\mathcal{T}(n_{\textit{te}}, d_u)$ with an unknown distribution~$d_u$, we first randomly sample a limited number of instances $\mathcal{G}_u^{'}$ from this task for validation (with the same distribution $d_u$ but not the test instances). Subsequently, we identify the decoder tailored to the distribution most proximate to the unknown distribution~$d_u$ based on the validation results as follows:
\begin{equation}\label{eq8}
    \theta_{\textit{de}}^*=\underset{i\in\{1,\dots, n_{\mathcal{D}}\}}{\arg \min} \operatorname{c}(\pi_i|\mathcal{G}_u^{'}\sim\mathcal{T}(n_{\textit{te}}, d_u), \theta_{\textit{en}}, \theta_{\textit{de}}^i),
\end{equation}
where $\pi_i$ denotes the tour produced by the $i$th decoder using greedy search. Finally, we employ the policy (i.e., the encoder $\theta_{\textit{en}}$ and the determined decoder $\theta_{\textit{de}}^*$) to solve VRPs from the unseen task $\mathcal{T}(n_{\textit{te}}, d_u)$.


Furthermore, considering the challenge of sampling from specific tasks, such as those within the TSPLIB dataset that comprise a few or only one single instance, we first input the test instance $\mathcal{G}_u$ within this unsamplable task into the encoder $\theta_{\textit{en}}$ to obtain the node features. Then, we repeat these node features $n_{\mathcal{D}}$ times and feed them into the decoders to produce $n_{\mathcal{D}}$ tours, each corresponding to a respective decoder. Finally, we select the optimal tour among these $n_{\mathcal{D}}$ tours as the output as follows:
\begin{equation}
\label{eq13}
\pi^* = \underset{i\in\{1,\dots, n_{\mathcal{D}}\}}{\arg \min} \operatorname{c}(\pi_i|\mathcal{G}_u, \theta_{\textit{en}}, \theta_{\textit{de}}^i).
\end{equation}

\begin{table*}[!t]
\centering
\small
\caption{Cross-size generalization performance of baseline models with (ours) and without (original) applying ESF during testing}\label{tab_log1}
\begin{threeparttable}
\begin{tabular}{ll|l|ccccccccc|cc}
\toprule

&\multirow{4}{*}{Model, $n_{\textit{tr}}$}                    &\multirow{4}{*}{inference}    & \multicolumn{3}{c}{$n_{\textit{te}}$=50}              & \multicolumn{3}{c}{$n_{\textit{te}}$=100}        & \multicolumn{3}{c|}{$n_{\textit{te}}$=200}           &\multicolumn{2}{c}{Avg. gain (\%) $\uparrow$  }               \\
& & & \multicolumn{2}{c}{O-gap (\%)$\downarrow$} & \multirow{3}{*}{\tabincell{c}{relative\\gain (\%)$\uparrow$}} & \multicolumn{2}{c}{O-gap (\%) $\downarrow$} & \multirow{3}{*}{\tabincell{c}{relative\\gain (\%)$\uparrow$}} & \multicolumn{2}{c}{O-gap (\%) $\downarrow$} & \multicolumn{1}{c|}{\multirow{3}{*}{\tabincell{c}{relative\\gain (\%)$\uparrow$}}}  & \multicolumn{1}{|c}{\multirow{3}{*}{\tabincell{c}{Scaling\\up}}}& \multirow{3}{*}{\tabincell{c}{Scaling\\down}}\\
& & & \multirow{2}{*}{\tabincell{c}{orig-\\inal}} & \multirow{2}{*}{\tabincell{c}{+ESF\\(ours)}} &  & \multirow{2}{*}{\tabincell{c}{orig-\\inal}} & \multirow{2}{*}{\tabincell{c}{+ESF\\(ours)}} &  & \multirow{2}{*}{\tabincell{c}{orig-\\inal}} & \multirow{2}{*}{\tabincell{c}{+ESF\\(ours)}} & & &  \\
& & & & & & & & & & & & & \\
  \midrule
 \multirow{12}{*}{\rotatebox{90}{TSP}}                     & \multirow{2}{*}{POMO, 50}  & G, no aug      & 0.24 & 0.24   & - & 1.28 & 1.01 & \textbf{+21.09} & 11.06 & 8.65 & \textbf{+21.79} & \multirow{4}{*}{\textbf{+23.88}} & \multirow{4}{*}{\textbf{+11.14}}  \\
                      &                            &G, $\times$8 aug          & 0.15 & 0.15 &- & 0.68 & 0.47 &\textbf{+30.88} & 8.95 & 6.79&\textbf{+24.13}  \\
                      & \multirow{2}{*}{POMO, 100} & G, no aug      & 0.47 & 0.39 & \textbf{+17.02} & 0.36 & 0.36 & -& 2.08 & 1.67 &\textbf{+19.71}   \\
                      &                            & G, $\times$8 aug          & 0.19 & 0.18 & \textbf{+5.26} & 0.13 & 0.13 &-  & 1.40 & 1.04 & \textbf{+25.71}   \\
                      \cmidrule{2-14}
                      & \multirow{2}{*}{AM, 50}    & G       & 1.78 & 1.78 &- & 4.94 & 4.83 &\textbf{+2.23} & 13.42 & 12.58 & \textbf{+6.26} & \multirow{4}{*}{\textbf{+15.64}} & \multirow{4}{*}{\textbf{+6.81}} \\
                      &                   & S, 1280 & 0.61 & 0.61 & -& 2.83 & 2.32 &\textbf{+17.73} & 20.51 & 10.72 & \textbf{+47.73} & & \\
                      & \multirow{2}{*}{AM, 100}   & G      & 4.45 & 4.38 & \textbf{+1.57} & 4.34 & 4.34 &- & 8.06 & 7.79 &\textbf{+3.35}   \\
                      &                            & S, 1280 & 2.49 & 2.19 &\textbf{+12.05} & 2.31 &2.31 &- & 6.69 & 5.58 & \textbf{+16.59} & &   \\
                      \cmidrule{2-14}
                    & \multirow{2}{*}{MDAM, 50}   & G      &0.75 & 0.75 &- & 3.02 & 2.90 &\textbf{+3.97} & 11.45 & 10.93 & \textbf{+4.54} &\multirow{4}{*}{\textbf{+2.74}} &\multirow{4}{*}{\textbf{+4.02}} \\
                      &                            & S, 30 & 0.17 & 0.17 &- & 0.77 & 0.75 &\textbf{+2.59} & 4.92 & 4.81 & \textbf{+2.24}    \\
                      & \multirow{2}{*}{MDAM, 100}   & G       &1.88 & 1.85 &\textbf{+1.59} & 2.18 & 2.18 & - & 5.70 & 5.55 & \textbf{+2.63} & &    \\
                      &                            & S, 30 & 0.31 & 0.29 &\textbf{+6.45} & 0.47 & 0.47 &- & 2.14 & 2.13 &\textbf{+0.47}    \\
                      \hline
\hline
\multirow{12}{*}{\rotatebox{90}{CVRP}} & \multirow{2}{*}{POMO, 50}  & G, no aug      & 1.94 & 1.94 &- & 5.26 & 5.79 & -10.07 & 11.28 & 15.39 & -36.43 & \multirow{4}{*}{-16.52} & \multirow{4}{*}{\textbf{+9.23}} \\
                      &                            &G, $\times$8 aug          & 0.96 & 0.96 & - & 3.77 & 4.19 & -11.14 & 9.55 & 13.27& -38.95 & &  \\
                      & \multirow{2}{*}{POMO, 100} & G, no aug      & 4.11 & 3.72 & \textbf{+9.49} & 1.85 & 1.85 &- & 4.53 & 4.62 & -1.99    \\
                      &                            &G, $\times$8 aug          & 2.23 & 2.03 & \textbf{+8.97} & 0.97 & 0.97 & - & 3.47 & 3.49 & -0.57    \\
                      \cmidrule{2-14}
                      & \multirow{2}{*}{AM, 50}    & G       & 6.01 & 6.01 &- & 8.42 & 8.30 &\textbf{+1.42} & 16.90 & 15.81 & \textbf{+6.45} & \multirow{4}{*}{\textbf{+9.44}} & \multirow{4}{*}{\textbf{+5.08}}  \\
                      &                            & S, 1280 & 2.58 & 2.58 &- & 4.35 & 4.19 &\textbf{+3.68} & 15.60 & 11.65 & \textbf{+25.32}  \\
                      & \multirow{2}{*}{AM, 100}   & G       & 8.75 & 8.54 &\textbf{+2.41} & 7.10 & 7.10 &- & 8.06 & 7.79 & \textbf{+3.35} & &  \\
                      &                            & S, 1280 & 4.13 & 3.81 & \textbf{+7.75} & 3.55 & 3.55 &- & 6.69 & 5.59 &\textbf{+16.44}  \\
                      \cmidrule{2-14}
                      & \multirow{2}{*}{MDAM, 50}   & G       &3.62 & 3.62 & - &5.87 & 5.90 &-0.51 &10.77 & 11.26 & -4.55 & \multirow{4}{*}{-3.55} &\multirow{4}{*}{\textbf{+3.61}} \\
                      &                            & S, 30 &1.36 & 1.36 & - &3.03 & 3.06 & -0.99 & 8.45 & 9.82 & -16.21 & &  \\
                      & \multirow{2}{*}{MDAM, 100}   & G       &6.65 & 6.60 &\textbf{+0.75} & 4.86 & 4.86 &- & 7.51 & 7.44 & \textbf{+0.93}   \\
                      &                            & S, 30 & 0.31 & 0.29 &\textbf{+6.45} & 2.50 & 2.50 &- & 4.97 & 4.97 &-0.00 & &    \\
                      \bottomrule
\end{tabular}
\begin{tablenotes}
\item[] Symbol ``O-gap'' denotes the optimality gap. Unless stated otherwise, the optimality gap is computed w.r.t. the exact solvers Concorde (for TSP) and LKH3 (for CVRP). The relative gain is obtained from the equation of $1-(\text{the result of baseline with ESF})/(\text{the result of baseline without ESF})$. Symbols ``G'' and ``S'' denote the greedy search mode and the sampling mode used in \cite{Kool2019}, respectively, and the value followed by symbol ``S'' denotes the number of samples. Symbol ``aug'' denotes the augment technique used in \cite{Kwon2020}. 
\end{tablenotes}
\end{threeparttable}
\end{table*}

In practice, our approach only introduces multiple light decoders corresponding to multiple predefined distribution patterns to achieve distribution generalization improvement. Compared to the testing process of the conventional methods \cite{Bi2022, Zhou2023}, our approach only introduces an extra sampling step for decoder determination on samplable tasks, or the additional computational resource of $n_{\mathcal{D}}-1$ light decoders on unsamplable tasks. This minimal increase in computational resource overhead is deemed acceptable in practical applications. See Table~\ref{tab_ds} for the extra inference time incurred by our approach.

\begin{table*}[!t]
\small
\centering
\caption{Large-scale generalization performance of baseline models ($n_{\textit{tr}}=100$) with (ours) and without (original) applying ESF during testing on TSPs}\label{tab_ls}
\begin{threeparttable}
\begin{tabular}{l|l|ccccccccc|c}
\toprule
\multirow{4}{*}{Model}                    &\multirow{4}{*}{inference}    & \multicolumn{3}{c}{$n_{\textit{te}}$=200}              & \multicolumn{3}{c}{$n_{\textit{te}}$=500}        & \multicolumn{3}{c|}{$n_{\textit{te}}$=1000}           &\multicolumn{1}{c}{\multirow{4}{*}{\tabincell{c}{Avg.\\gain (\%) $\uparrow$} }}               \\
& & \multicolumn{2}{c}{O-gap (\%)$\downarrow$} & \multirow{3}{*}{\tabincell{c}{relative\\gain (\%)$\uparrow$}} & \multicolumn{2}{c}{O-gap (\%) $\downarrow$} & \multirow{3}{*}{\tabincell{c}{relative\\gain (\%)$\uparrow$}} & \multicolumn{2}{c}{O-gap (\%) $\downarrow$} & \multicolumn{1}{c|}{\multirow{3}{*}{\tabincell{c}{relative\\gain (\%)$\uparrow$}}}  & \\
& & \multirow{2}{*}{\tabincell{c}{orig-\\inal}} & \multirow{2}{*}{\tabincell{c}{+ESF\\(ours)}} &  & \multirow{2}{*}{\tabincell{c}{orig-\\inal}} & \multirow{2}{*}{\tabincell{c}{+ESF\\(ours)}} &  & \multirow{2}{*}{\tabincell{c}{orig-\\inal}} & \multirow{2}{*}{\tabincell{c}{+ESF\\(ours)}} & &   \\
& & & & & & & & & & \\
                       \midrule
\multirow{2}{*}{LEHD}  & G                                         & 0.86 & 0.84 &\textbf{+2.32} & 1.56 & 1.47 & \textbf{+5.77} & 3.17 & 2.59 & \textbf{+18.30} & \multirow{2}{*}{\textbf{+11.06}}\\

                       & RRC, 10                                         & 0.42 & 0.36 &\textbf{+14.28} & 0.83 & 0.78 &\textbf{+6.02} & 1.98 & 1.59 &\textbf{+19.70} \\
                       \midrule
\multirow{2}{*}{BQ}    & G                                         & 0.62 & 0.57 &\textbf{+8.06} & 1.08 & 0.98 &\textbf{+9.26} & 2.10 & 1.71 &\textbf{+18.57} & \multirow{2}{*}{\textbf{+10.45}} \\
                       & S, 16                                          & 0.10 & 0.10 &\textbf{+0.01}  & 0.58 & 0.51 &\textbf{+12.07} & 1.36 & 1.16 &\textbf{+14.70}\\
                    \bottomrule
\end{tabular}
\begin{tablenotes}
\item[] Symbol ``RRC'' denotes the solution update strategy used in \cite{Luo2023}.
\end{tablenotes}
\end{threeparttable}
\end{table*}

\begin{table*}[!t]
\centering
\small
\caption{In/cross-size generalization performance of OMNI-VRP ($n_{\textit{tr}}\in [50, 200]$) with (ours) and without (original) applying ESF during training}\label{tab_log2}
\begin{threeparttable}
\begin{tabular}{llcccccc|c|cccccc|c}
\toprule
& \multicolumn{1}{l|}{\multirow{4}{*}{Inference}} & \multicolumn{6}{c|}{In-size optimality gap (\%)$\downarrow$ }         &\multirow{4}{*}{\tabincell{c}{Avg. \\gain\\(\%)$\uparrow$}}        & \multicolumn{6}{c|}{Cross-size optimality gap (\%)$\downarrow$ }      &\multirow{4}{*}{\tabincell{c}{Avg. \\gain\\(\%)$\uparrow$}}                       \\
& \multicolumn{1}{c}{}                           & \multicolumn{2}{|c}{$n_{\textit{te}}$=100}              & \multicolumn{2}{c}{$n_{\textit{te}}$=150}             & \multicolumn{2}{c|}{$n_{\textit{te}}$=200}            & & \multicolumn{2}{c}{$n_{\textit{te}}$=300}              & \multicolumn{2}{c}{$n_{\textit{te}}$=500}               & \multicolumn{2}{c|}{$n_{\textit{te}}$=1000} &               \\

& \multicolumn{1}{c|}{}   & \multirow{2}{*}{\tabincell{c}{orig-\\inal}} & \multirow{2}{*}{\tabincell{c}{+ESF\\(ours)}} & \multirow{2}{*}{\tabincell{c}{orig-\\inal}} & \multirow{2}{*}{\tabincell{c}{+ESF\\(ours)}} &\multirow{2}{*}{\tabincell{c}{orig-\\inal}} & \multirow{2}{*}{\tabincell{c}{+ESF\\(ours)}} & &\multirow{2}{*}{\tabincell{c}{orig-\\inal}}& \multirow{2}{*}{\tabincell{c}{+ESF\\(ours)}}&\multirow{2}{*}{\tabincell{c}{orig-\\inal}} & \multirow{2}{*}{\tabincell{c}{+ESF\\(ours)}} &\multirow{2}{*}{\tabincell{c}{orig-\\inal}}& \multirow{2}{*}{\tabincell{c}{+ESF\\(ours)}}\\
&\multicolumn{1}{c|}{}   & & & & & & & & & & & & & &\\
\midrule
\multicolumn{1}{c}{\multirow{4}{*}{\rotatebox{90}{TSP}}}  & \multicolumn{1}{c|}{G, no aug}                                      & 2.22 & 1.87 & 2.45 & 2.06 & 2.80 & \multicolumn{1}{c|}{2.40} & \multirow{4}{*}{\textbf{+18.38}} & 4.24 & 3.66  & 8.95 & 7.73  &   20.37 & 18.23   & \multirow{4}{*}{\textbf{+13.80}}           \\
\multicolumn{1}{c}{}                      & \multicolumn{1}{c|}{G, $\times$8 aug}                                         & 1.33 & 1.02 & 1.61 & 1.32 & 2.00 & \multicolumn{1}{c|}{1.67} & & 3.38 & 2.88 & 7.89 & 6.71  &   19.27& 17.05    &             \\
  & \multicolumn{1}{c|}{F, G, no aug}                                      & 2.03 & 1.63 & 2.36 & 1.98 & 2.86 & \multicolumn{1}{c|}{2.40} & & 4.11 & 3.49  & \multicolumn{2}{c}{-}  &   \multicolumn{2}{c|}{-}                 \\
\multicolumn{1}{c}{}                      & \multicolumn{1}{c|}{F, G, $\times$8 aug}                                         & 1.23 & 0.88 & 1.54 & 1.26 & 2.03 & \multicolumn{1}{c|}{1.66} & & 3.26 & 2.73 & \multicolumn{2}{c}{-}  &   \multicolumn{2}{c|}{-}   &             \\
                      \hline
\hline
\multicolumn{1}{c}{\multirow{4}{*}{\rotatebox{90}{CVRP}}} & \multicolumn{1}{c|}{G, no aug}                                       & 3.79 & 3.43 & 3.60 & 3.48 & 3.81 & \multicolumn{1}{c|}{3.78} & \multirow{4}{*}{\textbf{+4.51}} & 4.78 & 4.60 & 7.44 & 7.00 & 16.91 & 15.67 & \multirow{4}{*}{\textbf{+4.86}}\\
\multicolumn{1}{c}{}                      & \multicolumn{1}{c|}{G, $\times$8 aug}                                            & 2.78 & 2.48 & 2.81 & 2.67 & 3.11 & \multicolumn{1}{c|}{3.01} & & 4.07& 3.95  & 6.65& 6.24  & 15.13& 14.30 & \\    
 & \multicolumn{1}{c|}{F, G, no aug}                                       & 3.39& 3.22 & 3.53& 3.41 & 3.82& \multicolumn{1}{c|}{3.78} & & 4.60& 4.44 & \multicolumn{2}{c}{-} & \multicolumn{2}{c|}{-} & \\
\multicolumn{1}{c}{}                      & \multicolumn{1}{c|}{F, G, $\times$8 aug}                                            & 2.49& 2.36 & 2.75& 2.65 & 3.11& \multicolumn{1}{c|}{3.01} & & 3.95& 3.80  & \multicolumn{2}{c}{-}  & \multicolumn{2}{c|}{-} & \\   
\bottomrule
\end{tabular}
\begin{tablenotes}
\item[] Symbol ``F'' denotes the few-shot setting used in \cite{Zhou2023}. We are unable to evaluate VRPs of the size larger than 500 with the few-shot setting due to the constraint of the GPU memory.
\end{tablenotes}
\end{threeparttable}
\end{table*}

\section{Experimental Results}
We comprehensively evaluate the effectiveness of the proposed ESF and DS decoder on TSPs and CVRPs, respectively, using a computer equipped with an Intel(R) Xeon(R) Gold 6254 CPU and an NVIDIA RTX 3090 GPU. Our implementations of ESF and DS decoder are accessible online\footnote{URL:~\href{https://github.com/xybFight/VRP-Generalization}{https://github.com/xybFight/VRP-Generalization}}

\subsection{Effectiveness Analysis of ESF}
To demonstrate the effectiveness of the proposed ESF during testing and training (refer to Figures~\ref{sf1} and \ref{sf2}, respectively), we conduct two sets of experiments, each time using a test dataset comprising 10,000 randomly generated VRP instances, following the same approach adopted by \cite{Kwon2020}.

Firstly, we evaluate the size generalization performance of applying ESF during the test process. To this end, we select three representative models (AM \cite{Kool2019}, POMO \cite{Kwon2020}, and MDAM \cite{Xin2021}) as the baselines. Specifically, we assess the original cross-size generalization performance of these publicly available pre-trained models, and compare it with that of these models incorporating our proposed ESF: $\operatorname{log}_{n_{\textit{tr}}}n_{\textit{te}}$. Considering that these models are trained on VRPs of two different sizes $n_{\textit{tr}}\in\{50, 100\}$, we conduct tests on VRPs of three different sizes $n_{\textit{te}}\in\{50, 100, 200\}$ to evaluate the ESF's capability to scaling up ($n_{\textit{tr}}=50\rightarrow n_{\textit{te}}\in\{100, 200\}$, $n_{\textit{tr}}=100\rightarrow n_{\textit{te}}\in\{ 200\}$) and scaling down ($n_{\textit{tr}}=100\rightarrow n_{\textit{te}}\in \{50\}$). We present the results conducted on TSPs and CVRPs in Table~\ref{tab_log1}. As shown, the baseline models consistently exhibit performance degradation in both up-scaling and down-scaling scenarios, underscoring the critical need for investigating comprehensive size generalization methods. For TSPs, our proposed ESF effectively enhances the size generalization capabilities of these models, demonstrating an average relative improvement of $14.09\%$ and $7.32\%$ for scaling up and scaling down instances, respectively. For CVRPs, ESF manifests its effectiveness primarily in down-scaling scenarios (with an average relative improvement of $5.97\%$), while potentially compromising the model's up-scaling capability (with an average relative drop of $3.54\%$). We posit this phenomenon arises from the following two reasons: 1) ESF ($\operatorname{log}_{n_{\textit{tr}}}n_{\textit{te}}>1$) increases the likelihood of selecting the depot node, leading to more sub-tours within the solution and consequently increasing the overall solution length. 2) The multiple selections of the depot node in CVRP leads to a solution sequence length exceeding the test size of $n_{\textit{te}}$ and remaining non-constant for different instances. Introducing a fixed scaling factor may exacerbate unfamiliar changes in the attention weight, subsequently resulting in performance degradation. A potential solution to these issues is to incorporate an auxiliary network to dynamically adjust the scaling factor's value for CVRPs. We leave this approach for future research because in this work, we focus on the generic generalization method across varying VRP variants. In subsequent experiments involving the incorporation of ESF during testing, we only present the experimental results on TSPs, excluding CVRPs.

Furthermore, considering the large-scale VRPs (e.g., $n_{\textit{te}}=1000$) are akin to real-world cases than the smaller counterparts (e.g., $n_{\textit{te}}=50$), we select two recent models (LEHD \cite{Luo2023}, BQ \cite{Drakulic2023}), both trained on instances with size $n_{\textit{tr}}=100$ as baselines to further evaluate the effectiveness of ESF on TSPs. It is worthy highlighting that both of these models have showcased the state-of-the-art (SOTA) results in solving large-scale TSPs. As the results presented in Table~\ref{tab_ls}, ESF further enhances the generalization performance of these two models for large-scale TSPs, with an average relative improvement of $10.76\%$, thereby highlighting the efficacy of our proposed ESF.

\begin{figure*}[!t]
\centering
\subfigure[$\textit{d}_\textit{U}$, Uniform]{\includegraphics[width=.33\columnwidth]{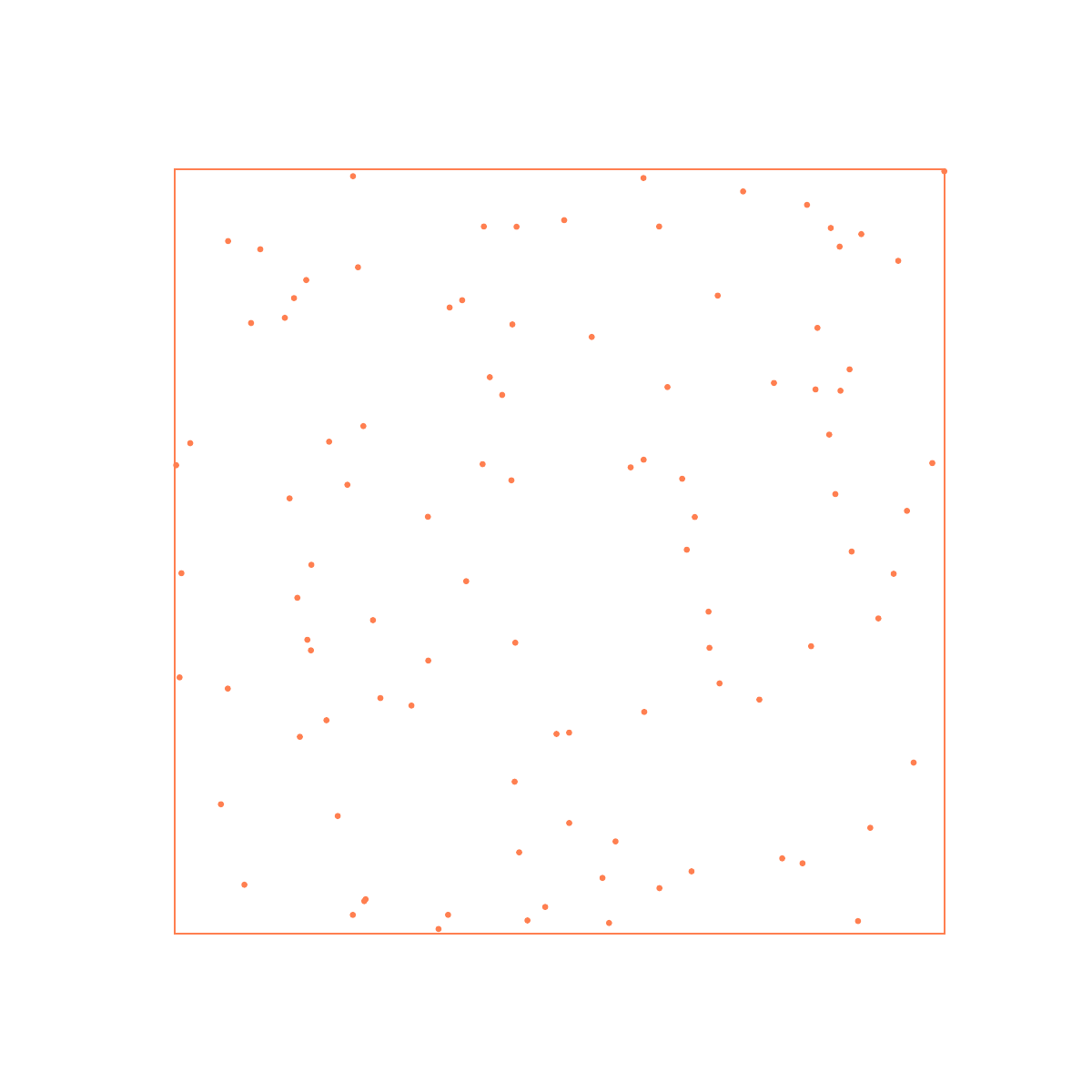}\label{figu}}
\subfigure[$\textit{d}_\textit{C}$, Cluster]{\includegraphics[width=.33\columnwidth]{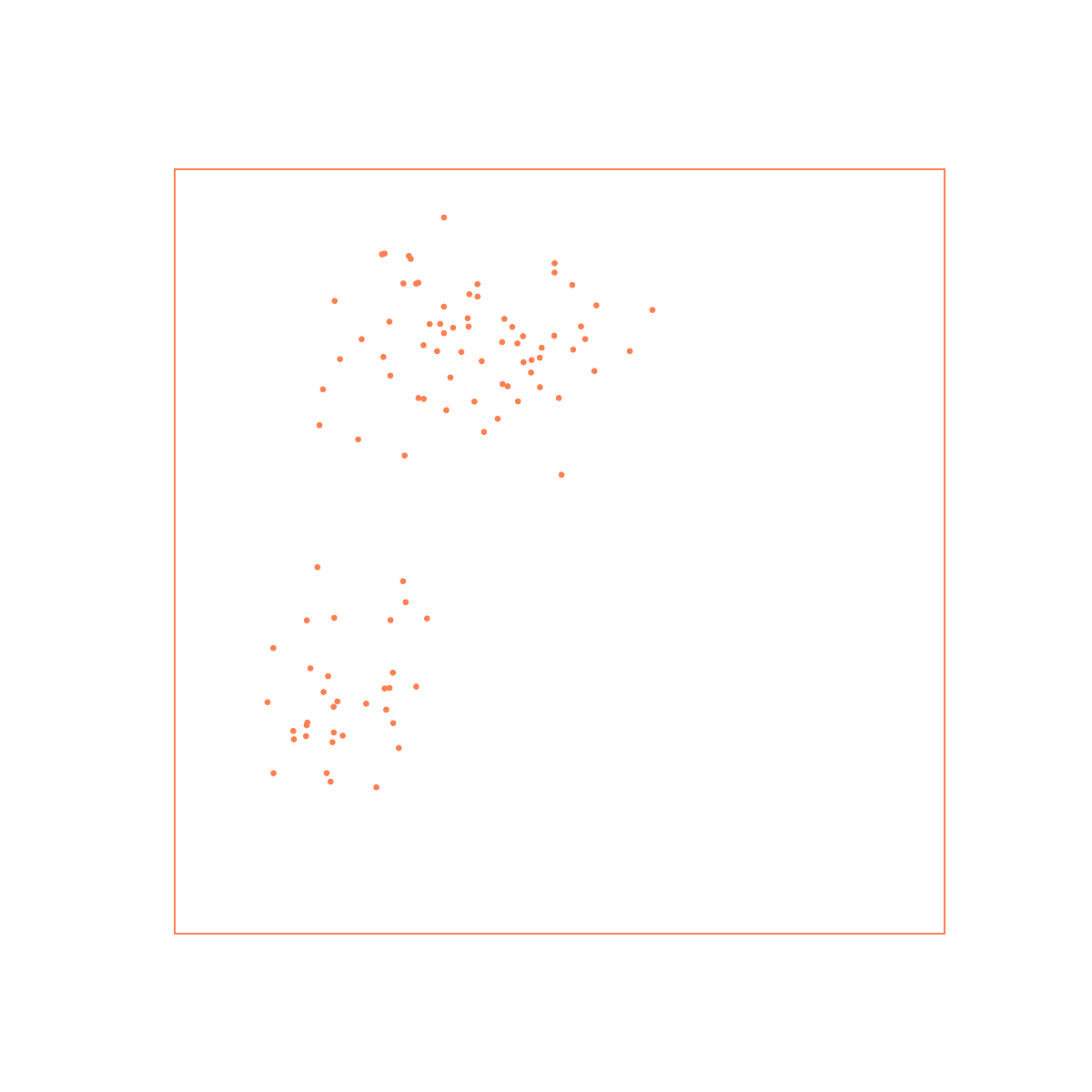}\label{figu}}
\subfigure[$\textit{d}_\textit{M}$, Mixed]{\includegraphics[width=.33\columnwidth]{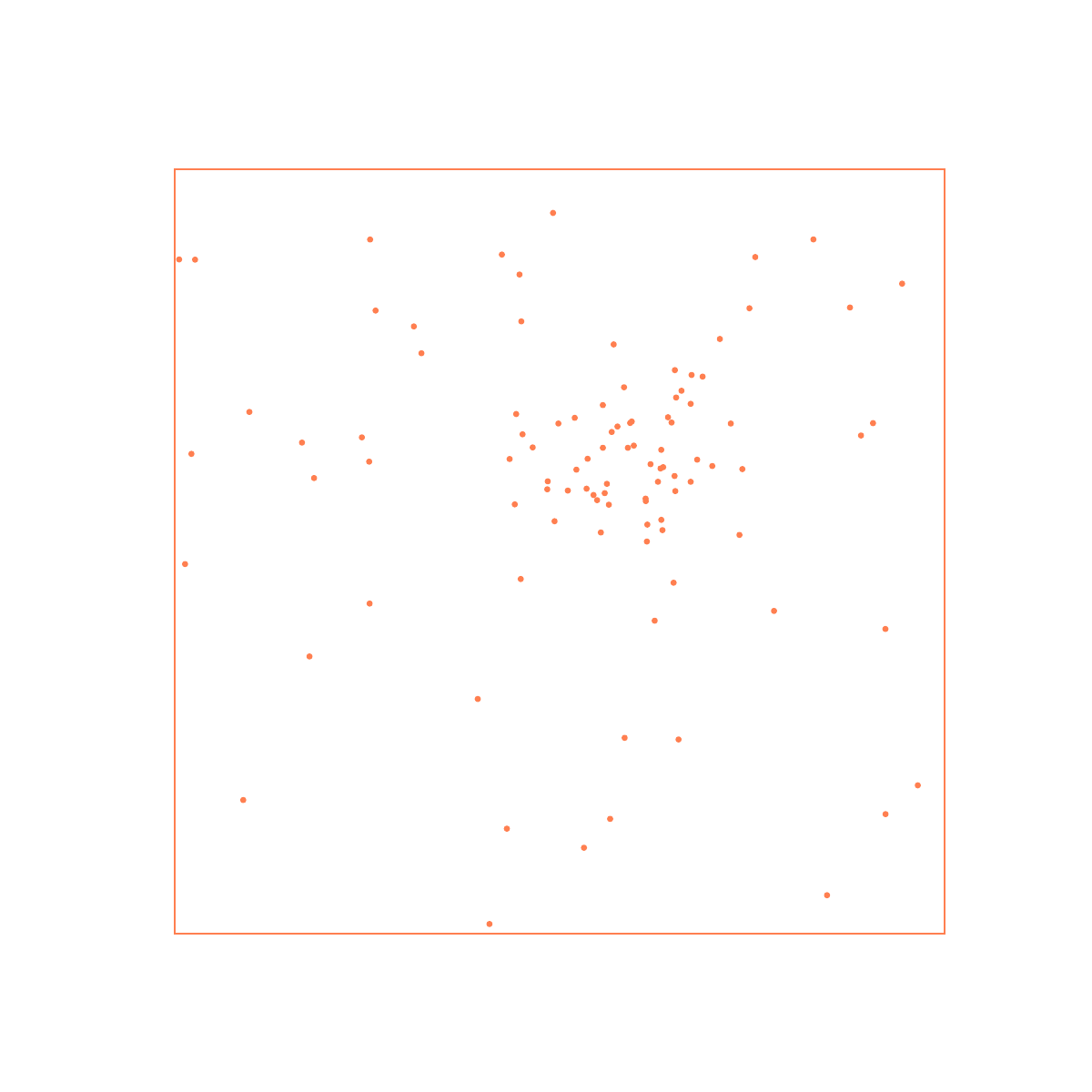}\label{figc}}
\subfigure[$\textit{d}_\textit{I}$, Implosion]{\includegraphics[width=.33\columnwidth]{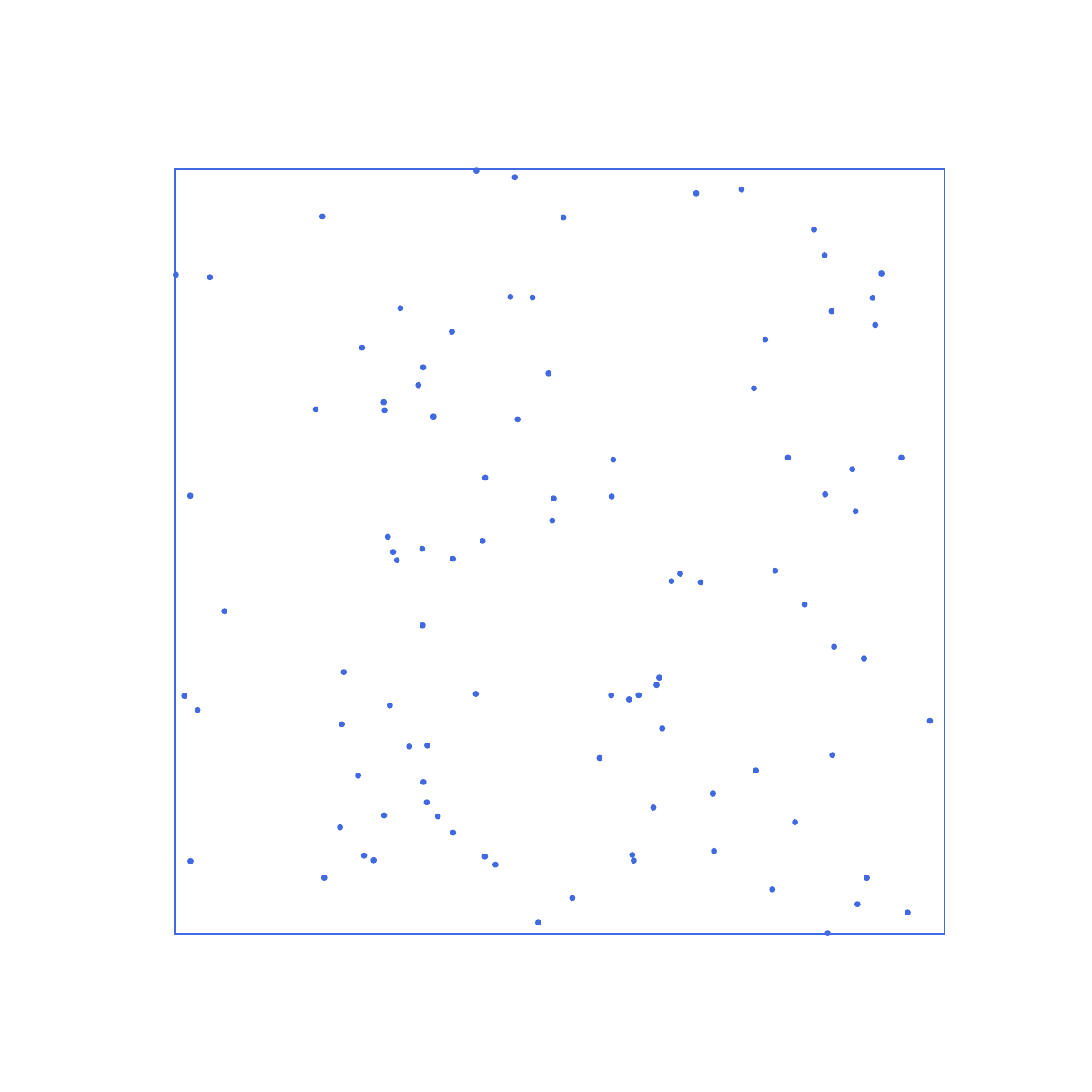}\label{figc}}
\subfigure[$\textit{d}_\textit{Eo}$, Explosion]{\includegraphics[width=.33\columnwidth]{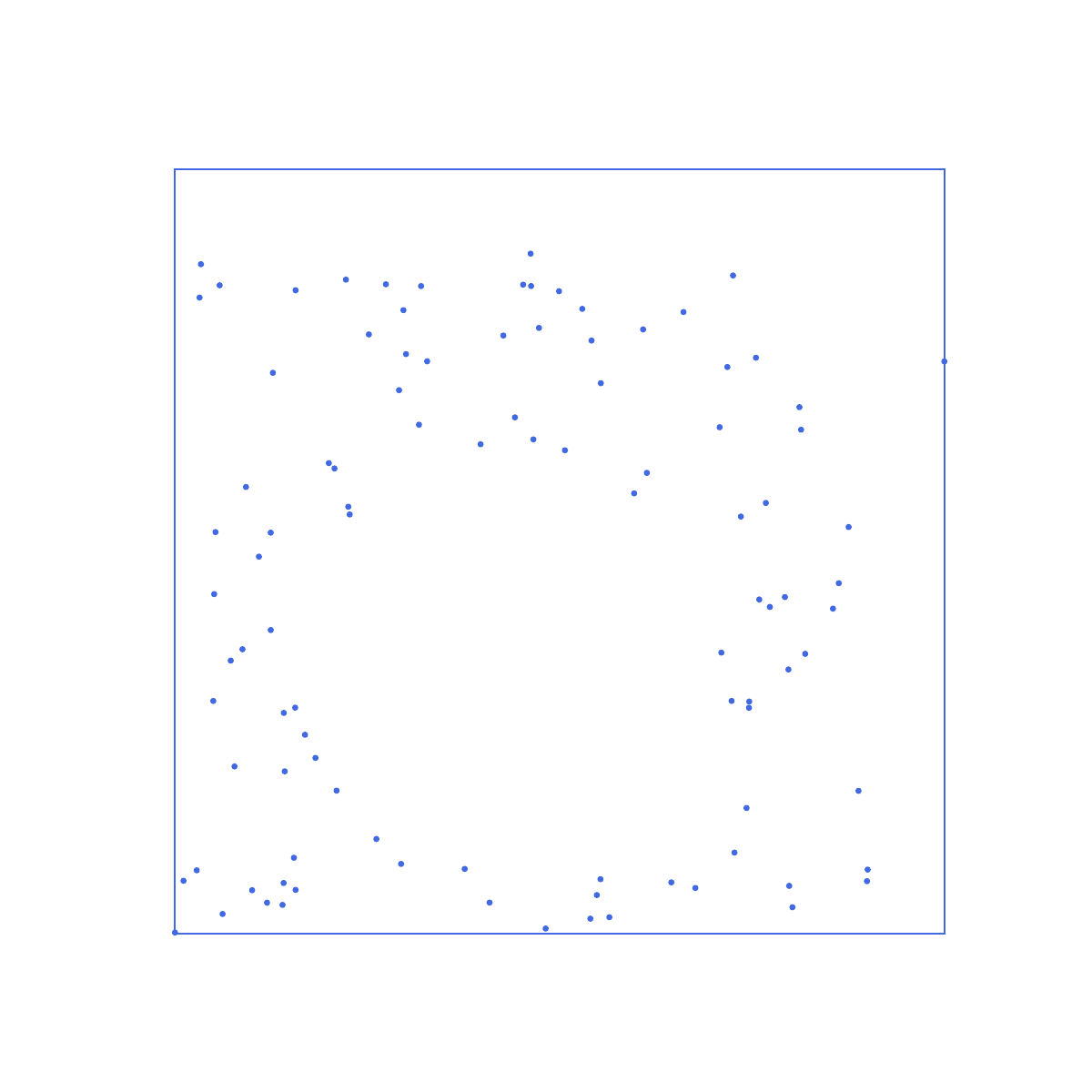}\label{figc}}
\subfigure[$\textit{d}_\textit{Ea}$, Expansion]{\includegraphics[width=.33\columnwidth]{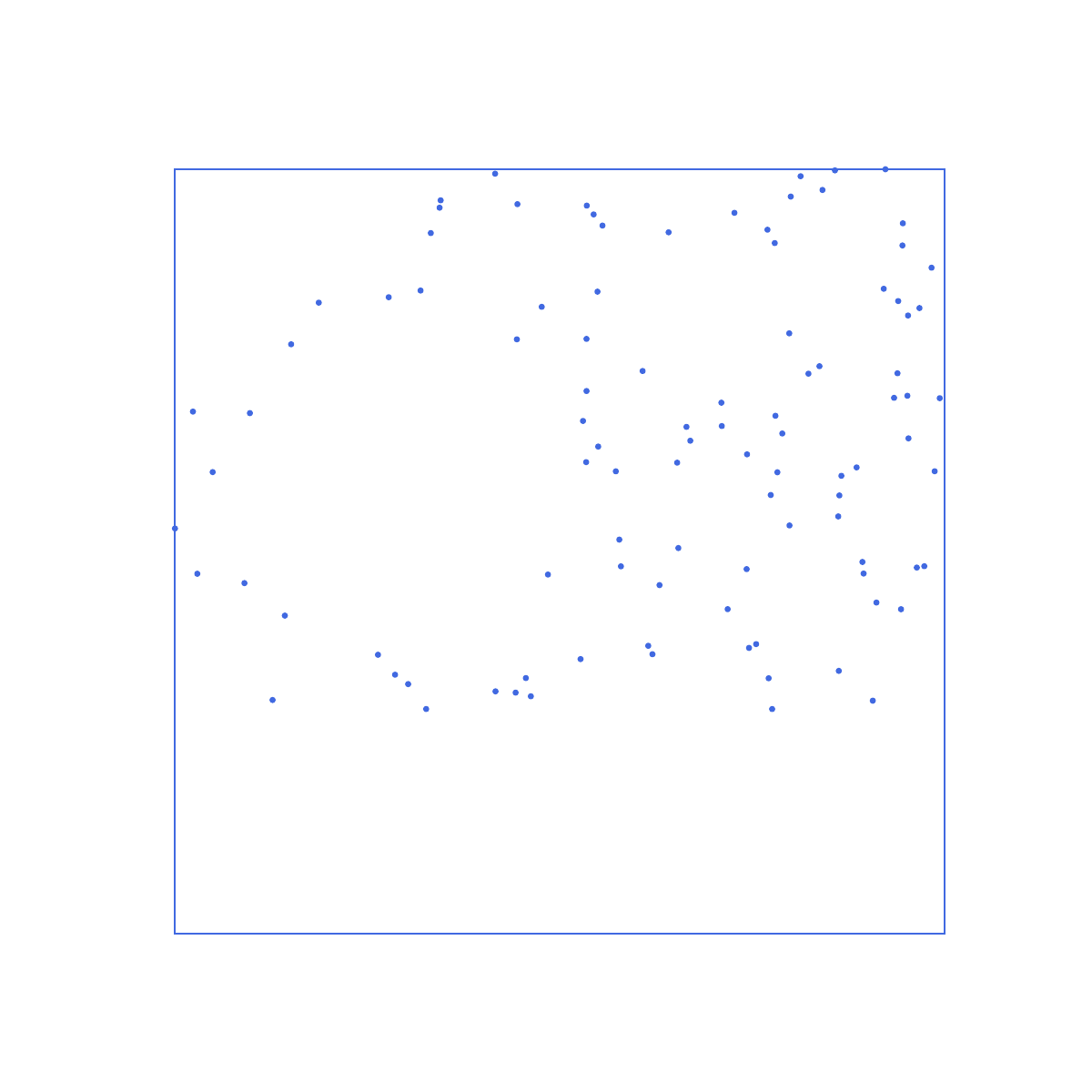}\label{figc}}
\caption{Visualization of VRP instances with various distribution patterns. We consider instances with distribution patterns (a)-(c) for training and all distribution patterns (a)-(f), as well as unseen benchmarking datasets (e.g., TSPLIB and CVRPLIB) for testing, following \cite{Bi2022}.}\label{fig4}
\end{figure*}

\begin{table*}[!t]
\small
\centering
\caption{In/cross-distribution generalization performance of POMO, AMDKD, OMNI-VRP and our proposed DS decoder}\label{tab_ds}
\begin{threeparttable}
\resizebox{2.1\columnwidth}{!}{
\begin{tabular}{lllccccccccccccccc}
\toprule
\multicolumn{1}{l}{}  & \multicolumn{2}{l|}{\multirow{3}{*}{Model, $d_i$}}    & \multicolumn{7}{c|}{$n_{\textit{te}}=50$}
& \multicolumn{7}{c}{$n_{\textit{te}}$=100}                                              & \multicolumn{1}{|c}{\multirow{3}{*}{{\tabincell{l}{Avg.\\gap}}}} \\
\multicolumn{3}{l}{}    & \multicolumn{6}{|c}{Optimality gap (\%)$\downarrow$} & \multicolumn{1}{c|}{\multirow{2}{*}{\tabincell{l}{Time\\(min)}}}    & \multicolumn{6}{c}{Optimality gap (\%)$\downarrow$}    & \multicolumn{1}{c|}{\multirow{2}{*}{\tabincell{l}{Time\\(min)}}}               \\
  &          &  & \multicolumn{1}{|c}{$\textit{d}_\textit{U}$} & $\textit{d}_\textit{C}$ & $\textit{d}_\textit{M}$ & $\textit{d}_\textit{I}$ & $\textit{d}_\textit{Eo}$ & \multicolumn{1}{c}{$\textit{d}_\textit{Ea}$} &  \multicolumn{1}{c|}{} & $\textit{d}_\textit{U}$ & $\textit{d}_\textit{C}$ & $\textit{d}_\textit{M}$ & $\textit{d}_\textit{I}$ & $\textit{d}_\textit{Eo}$ & \multicolumn{1}{c}{$\textit{d}_\textit{Ea}$} &  \multicolumn{1}{c|}{} &\multicolumn{1}{c}{}    \\
\midrule
\multirow{10}{*}{{\rotatebox{90}{TSP}}}  & \multicolumn{2}{l|}{POMO, $\textit{d}_\textit{U}$}   &0.04         &0.42         &0.21       &0.04           &0.03      &\multicolumn{1}{c|}{0.09}  & \multicolumn{1}{c|}{0.28}        &\textbf{0.17}           &1.97         &0.92         &0.17       & 0.20          &\multicolumn{1}{c|}{0.64}        & \multicolumn{1}{c|}{1.12}  &0.41     \\
                      & \multicolumn{2}{l|}{POMO, $\textit{d}_\textit{C}$}   &0.09         &0.07         &0.21       &0.08           &0.10      &\multicolumn{1}{c|}{0.23}   & \multicolumn{1}{c|}{0.28}       &0.41           &0.29         &0.83         &0.38       &0.33               &\multicolumn{1}{c|}{0.85}     &\multicolumn{1}{c|}{1.12}      &0.32     \\
                      & \multicolumn{2}{l|}{POMO, $\textit{d}_\textit{M}$}   &0.05         &0.17         &0.08       &0.08           &0.07      &\multicolumn{1}{c|}{0.12}     & \multicolumn{1}{c|}{0.28}     &0.77           &1.16         &0.34         &0.76       &0.57               &\multicolumn{1}{c|}{1.06}     & \multicolumn{1}{c|}{1.12}     &0.44     \\
                      & \multicolumn{2}{l|}{AMDKD, $\textit{d}_\textit{U}, \textit{d}_\textit{C}, \textit{d}_\textit{M}$} &0.05         &0.05         &0.09       &0.05           &0.03      &\multicolumn{1}{c|}{0.08}     &  \multicolumn{1}{c|}{0.28}    &0.34           & 0.35       &0.41           &0.34  &0.28 &\multicolumn{1}{c|}{0.65}    &   \multicolumn{1}{c|}{1.12}  &0.23     \\

                        & \multicolumn{2}{l|}{OMNI-VRP, $\textit{d}_\textit{U}, \textit{d}_\textit{G}$}   &0.84         &0.80         &1.82       &0.81           &0.63      &\multicolumn{1}{c|}{1.00}   &    \multicolumn{1}{c|}{0.28}   &1.27           &1.30         &2.24         &1.23       &0.99               &\multicolumn{1}{c|}{1.45}   &    \multicolumn{1}{c|}{1.12}    &1.20     \\

                    & \multicolumn{2}{l|}{OMNI-VRP$^*$, $\textit{d}_\textit{U}, \textit{d}_\textit{G}$}   &0.57         &0.52         &1.40       &0.54           &0.42      &\multicolumn{1}{c|}{0.76}   &  \multicolumn{1}{c|}{0.28}     &0.97           &1.04         &1.82         &0.94       &0.76               &\multicolumn{1}{c|}{1.23}   &   \multicolumn{1}{c|}{1.12}     &0.92     \\
                      
                      & \multicolumn{2}{l|}{DS, $\textit{d}_\textit{U}, \textit{d}_\textit{C}, \textit{d}_\textit{M}$, (ours)}      &\textbf{0.02}         &\textbf{0.03}         &\textbf{0.04}       &\textbf{0.02}           &\textbf{0.01}      &\multicolumn{1}{c|}{\textbf{0.03}}   & \multicolumn{1}{c|}{0.29}       &\textbf{0.17}           &\textbf{0.25}         &\textbf{0.23}         &\textbf{0.17}       &\textbf{0.13}              &\multicolumn{1}{c|}{\textbf{0.38}}   &      \multicolumn{1}{c|}{1.13}  &\textbf{0.12}     \\
                      \cmidrule{2-18}
                    & \multicolumn{1}{l}{\multirow{3}{*}{\tabincell{l}{DS with $n_{\mathcal{D}}=3$\\ decoders, (ours)}}} & \multicolumn{1}{l|}{min}    &0.01         &0.02         &0.03       &0.02           &0.01      &\multicolumn{1}{c|}{0.02}   &  \multicolumn{1}{c|}{}      &0.15           &0.21         &0.20         &0.15       &0.11              &\multicolumn{1}{c|}{0.31}   & \multicolumn{1}{c|}{}       &0.10  \\
                    &  & \multicolumn{1}{l|}{max}    &0.03         &0.05         &0.06       &0.03           &0.02      &\multicolumn{1}{c|}{0.05}   & \multicolumn{1}{c|}{0.32}        &0.21           &0.31         &0.27         &0.20       &0.16              &\multicolumn{1}{c|}{0.45}   &    \multicolumn{1}{c|}{1.21}    &0.15  \\
                     &  & \multicolumn{1}{l|}{mean}    &0.02         &0.03         &0.04       &0.02           &0.01      &\multicolumn{1}{c|}{0.03}     & \multicolumn{1}{c|}{}     &0.18           &0.26         &0.23         &0.18       &0.13              &\multicolumn{1}{c|}{0.38}    &   \multicolumn{1}{c|}{}    &0.13  \\
                                          \hline
\hline
\multirow{10}{*}{{\rotatebox{90}{CVRP}}} &\multicolumn{2}{l|}{POMO, $\textit{d}_\textit{U}$} &0.81         &1.53         &1.07       &0.81          &0.86           &\multicolumn{1}{c|}{1.00}   & \multicolumn{1}{c|}{0.33}      &\textbf{0.97}         &2.36       &1.32           &\textbf{0.98}      &1.10           &\multicolumn{1}{c|}{1.39}   &  \multicolumn{1}{c|}{1.19}      &1.18         \\
                      &\multicolumn{2}{l|}{POMO, $\textit{d}_\textit{C}$}   &1.17         &0.93         &1.07       &1.18           &1.13           &\multicolumn{1}{c|}{1.20}  &  \multicolumn{1}{c|}{0.33}      &1.45         &\textbf{1.30}       &1.22           &1.47      &1.48           &\multicolumn{1}{c|}{1.55}      &  \multicolumn{1}{c|}{1.19}    &1.26     \\
                      &\multicolumn{2}{l|}{POMO, $\textit{d}_\textit{M}$}   &1.22         &1.34         &0.85       &1.22          &1.16           &\multicolumn{1}{c|}{1.20}   &  \multicolumn{1}{c|}{0.33}     &1.90         &2.09       &0.97           &1.91      &1.83           &\multicolumn{1}{c|}{1.78}     & \multicolumn{1}{c|}{1.19}     &1.46     \\
                      &\multicolumn{2}{l|}{AMDKD, $\textit{d}_\textit{U}, \textit{d}_\textit{C}, \textit{d}_\textit{M}$}   
                      &0.86         &1.03         &0.90       &0.86              &0.90          &\multicolumn{1}{c|}{1.00}   &  \multicolumn{1}{c|}{0.33}     &1.08         &1.40       &1.01           &1.08      &1.28           &\multicolumn{1}{c|}{1.44} &       \multicolumn{1}{c|}{1.19}  &1.07 
                      \\
                    & \multicolumn{2}{l|}{OMNI-VRP, $\textit{d}_\textit{U}, \textit{d}_\textit{G}$}   &3.37         &2.78         &3.29       &3.36           &3.14      &\multicolumn{1}{c|}{2.82}      & \multicolumn{1}{c|}{0.33}    &2.12           &2.22         &2.24         &2.26       &2.28               &\multicolumn{1}{c|}{2.10}  &   \multicolumn{1}{c|}{1.19}      &2.67     \\

                    & \multicolumn{2}{l|}{OMNI-VRP$^*$, $\textit{d}_\textit{U}, \textit{d}_\textit{G}$}   &2.57         &2.01         &2.62       & 2.55          &2.41      &\multicolumn{1}{c|}{2.29}   &  \multicolumn{1}{c|}{0.33}      &1.93           &1.86         &1.79         &1.94      &1.94               &\multicolumn{1}{c|}{1.85}      &  \multicolumn{1}{c|}{1.19}   &2.15     \\
                      
                      &\multicolumn{2}{l|}{DS, $\textit{d}_\textit{U}, \textit{d}_\textit{C}, \textit{d}_\textit{M}$, (ours)}       &\textbf{0.72}         &\textbf{0.72}         &\textbf{0.65}       &\textbf{0.71}             &\textbf{0.70}           &\multicolumn{1}{c|}{\textbf{0.71}}    &  \multicolumn{1}{c|}{0.34}    &1.01         &1.39       &\textbf{0.73}           &1.03      &\textbf{1.08}           &\multicolumn{1}{c|}{\textbf{1.21}}  & \multicolumn{1}{c|}{1.21}        &\textbf{0.89}     \\
                                            \cmidrule{2-18}
                    & \multicolumn{1}{l}{\multirow{3}{*}{\tabincell{l}{DS with $n_{\mathcal{D}}=3$\\ decoders, (ours)}}} & \multicolumn{1}{l|}{min}    &0.60         &0.58         &0.52       &0.59           &0.58      &\multicolumn{1}{c|}{0.56}    & \multicolumn{1}{c|}{}       &0.82           &1.20         &0.58         &0.84       &0.87              &\multicolumn{1}{c|}{0.99}  &   \multicolumn{1}{c|}{}      &0.73  \\
                    &  & \multicolumn{1}{l|}{max}    &0.88         &0.87         &0.79       &0.87           &0.86      &\multicolumn{1}{c|}{0.85}  &   \multicolumn{1}{c|}{0.39}       &1.25           &1.67         &0.94         &1.28       &1.32              &\multicolumn{1}{c|}{1.43}   &   \multicolumn{1}{c|}{1.35}     &1.08  \\
                     &  & \multicolumn{1}{l|}{mean}    &0.74         &0.72         &0.65       &0.73           &0.72      &\multicolumn{1}{c|}{0.70}  &   \multicolumn{1}{c|}{}       &1.04           &1.43         &0.76         &1.06       &1.10              &\multicolumn{1}{c|}{1.21}    &  \multicolumn{1}{c|}{}     &0.91  \\
\bottomrule
\end{tabular}}
\begin{tablenotes}
\item[]Symbol $^*$ denotes the model applied with our proposed ESF during training and $\textit{d}_\textit{G}$ denotes the Gaussian distribution used in \cite{Zhou2023, Xin2022}.
\end{tablenotes}
\end{threeparttable}
\end{table*}

Secondly, we evaluate the size generalization performance of applying ESF during the training process. To this end, we select the SOTA model OMNI-VRP \cite{Zhou2023} trained on VRP instances with sizes $n_{\textit{tr}}$ ranging from $50$ to $200$, as the baseline for evaluation. Specifically, we train OMNI-VRP exactly according to its default settings, while applying the proposed ESF: 
$\operatorname{log}_{n_{b}}n_{\textit{tr}}$. Subsequently, we evaluate both the in-size ($n_{\textit{te}}\in\{100, 150, 200\}$) and the cross-size ($n_{\textit{te}}\in\{300, 500, 1000\}$) generalization performance of the pre-trained OMNI-VRP with applying ESF: $\operatorname{log}_{n_{b}}n_{\textit{te}}$, and compare it with that of the original pre-trained OMNI-VRP (i.e., without ESF). The results presented in Table~\ref{tab_log2} demonstrate that ESF further enhances the size generalization performance of OMNI-VRP on both TSPs and CVRPs. Specifically, regarding TSPs, there is an average relative improvement of $18.38\%$ and $13.8\%$ on in-size and cross-size generalization performance, respectively. For CVRPs, there is an average relative improvement of $4.51\%$ and $4.86\%$ on in-size and cross-size generalization performance, respectively. This observation that ESF enhances the model's generalization performance for CVRPs, as opposed to the results presented in Table~\ref{tab_log1}, can be attributed to the following reason: ESF enhances the model's capability to learn size-related aspects by dynamically adjusting the attention weight for CVRPs of different sizes during training, thus enhancing the model's ability to effectively generalize across varying sizes.

These experimental results demonstrate the effectiveness of ESF in facilitating size generalization, particularly noteworthy for its plug-and-play nature.

\subsection{Effectiveness Analysis of the DS Decoder}
To demonstrate the effectiveness of our proposed DS decoder, we apply it to POMO \cite{Kwon2020}, which is one of the most representative neural VRP solvers. Considering the SOTA cross-distribution generalization performance demonstrated by AMDKD \cite{Bi2022}, which trained POMO using VRPs of multiple distribution patterns, we follow their environmental settings exactly for both training and testing. Specifically, we employ Uniform ($\textit{d}_\textit{U}$), Cluster ($\textit{d}_\textit{C}$) and Uniform-Cluster mixed~($\textit{d}_\textit{M}$) distribution patterns for training, i.e., $n_{\mathcal{D}}=3$, while employing Implosion ($\textit{d}_\textit{I}$), Explosion ($\textit{d}_\textit{Eo}$) Expansion ($\textit{d}_\textit{Ea}$) and the three training distribution patterns for testing. In Figure~\ref{fig4}, we present a visualization of VRP instances with these distribution patterns. The detailed procedures for generating VRPs of these distribution patterns follow AMDKD exactly. To ensure a fair comparison, we align the number of training iterations with those set in AMDKD. The learning rate of Adam optimizer is set to 1e--4, consistent with AMDKD. During testing, we adopt the greedy rollout with $\times$8 augments \cite{Kwon2020} for our method and all baseline models. All experiments use test datasets comprising 10,000 randomly generated VRP instances per distribution pattern.

\begin{table*}[!t]
\small
\centering
\caption{Generalization performance of baseline models on TSPLIB}\label{TSPLIB}
\begin{threeparttable}
\begin{tabular}{l|l|cccccc}
\toprule
\multirow{3}{*}{Size} & \multirow{3}{*}{\#} & \multicolumn{6}{c}{Optimality gap (\%) $\downarrow$}   \\
 & & \multirow{2}{*}{OMNI-VRP}& \multirow{2}{*}{OMNI-VRP$^*$}& \multirow{2}{*}{AMDKD} & \multirow{2}{*}{AMDKD$^\dag$} & \multicolumn{2}{c}{Ours} \\

 & & & & & & DS & DS$^\dag$\\
 \midrule
 52-99 & 6 & 2.56& 1.80&  0.53& 0.51&  0.488$\pm$0.000& \textbf{0.485$\pm$0.004}\\
 100-150 & 17& 1.36& 1.12& 0.76& 0.74&  0.468$\pm$0.017 & \textbf{0.412$\pm$0.014} \\
 151-200& 6& 2.34& 2.22& 2.64& \textbf{2.19} & 3.526$\pm$0.129 & 2.480$\pm$0.089 \\
 \midrule
 All& 29& 1.81& 1.48& 1.10& 0.99& 1.104$\pm$0.030 & \textbf{0.856$\pm$0.026 }\\
 \bottomrule
\end{tabular}
\begin{tablenotes}
\item[] Symbol \# denotes the number of instances in the corresponding set and symbol $^\dag$ denotes the model with applying our proposed ESF during testing. 
\end{tablenotes}
\end{threeparttable}
\end{table*}

Table~\ref{tab_ds} presents the distribution generalization performance of the POMO model trained using VRPs of various distribution patterns ($\textit{d}_\textit{U}$, $\textit{d}_\textit{C}$, $\textit{d}_\textit{M}$), AMDKD, OMNI-VRP, and our proposed DS decoder-based method. As indicated by the average optimality gap (i.e., the rightmost column in Table~\ref{tab_ds}), the POMO model trained on a specific distribution pattern exhibits a limited generalization capability to other distribution patterns, resulting in suboptimal overall performance. For instance, POMO trained using TSPs of the Uniform distribution $\textit{d}_\textit{U}$ exhibits an optimality gap of $0.17\%$ on TSP-100 test instances of the Uniform distribution~$\textit{d}_\textit{U}$, whereas the average optimality gap increases to $0.68\%$ (relative drop: $\frac{0.68\%}{0.17\%}-1=300\%$) in test instances of the six distribution patterns. AMDKD effectively alleviates this issue by simultaneously training VRP instances with multiple distribution patterns. Nonetheless, this method leads to the model's deterioration, as observed in the comparative results of AMDKD and POMO trained on a single distribution pattern when solving VRPs from that specific distribution pattern. For instance, when solving TSP-100 test instances of the Uniform distribution $\textit{d}_\textit{U}$,  POMO trained using TSPs of the Uniform distribution $\textit{d}_\textit{U}$ exhibits an optimality gap of $0.17\%$, while AMDKD shows a worse optimality gap of $0.34\%$ (relative drop: $\frac{0.34\%}{0.17\%}-1=100\%$). Notably, the case of model deterioration in AMDKD constitutes $92\%$ of the total comparative results, indicating that the distribution space represented by the model struggles to adequately cover the space of the three training distribution patterns. In contrast, our method not only improves the performance of cross-distribution generalization (with  average relative improvements of $1-\frac{0.51\%}{0.79\%}=35.44\%$ and $1-\frac{0.51\%}{0.65\%}=22.31\%$ for POMO and AMDKD, respectively) but also effectively alleviates the model's deterioration, with the case of model deterioration constitutes $16.7\%$ of the total comparative results. Furthermore, the additional inference time incurred by our approach is minimal, amounting to only 0.29-0.28=0.01 minute. We deem that this minimal increase in computational resource overhead is acceptable in practical applications. Additionally, the implementation of AMDKD requires $n_{\mathcal{D}}=3$ pretrained teacher models for knowledge distillation, which undoubtedly increases the computational burden, while our method does not necessitate this additional step. For instance, considering TSP-100, training $n_{\mathcal{D}}=3$ teacher models for AMDKD on our computer takes approximately $n_{\mathcal{D}}\times5=15$ days.

Furthermore, given the stochasticity introduced by our method when selecting a single decoder through randomly sampling a few validation instances (see the context of (\ref{eq8})), we apply all the $n_{\mathcal{D}}=3$ decoders during testing for subsequent statistical analysis (see the context of (\ref{eq13})). The DS with $n_{\mathcal{D}}=3$ decoders setting leverages all 3 decoders, producing 3 solutions, which are then evaluated to return the max, mean, and min values of the optimality gap. The results, as presented in Table~\ref{tab_ds}, demonstrate that under the identical experimental setting, the performance of AMDKD is inferior to the worst~(max), average~(mean), and optimal~(min) cases of these decoders, except for CVRPs, where AMDKD outperforms only the worst case of our method by a slight margin. This result can be attributed to the DS decoder mitigates conflicts between instances of different distribution patterns, thereby improving the lower bound of the model performance. Moreover, the optimal case of the DS decoder consistently outperforms POMO trained on a certain distribution pattern when solving instances of that distribution pattern, i.e., no case of model deterioration in the overall comparative results. This result can be attributed to the employment of a shared encoder in our method, aiming to identify a uniform representation for instances of different distribution patterns. This approach improves the upper bound of model performance by considering distribution patterns beyond a specific distribution pattern as regular terms.

These experimental results demonstrate the effectiveness of the DS decoder in facilitating distribution generalization. Notably, our proposed DS decoder sidesteps the implementation of complex generalization algorithms, rendering it highly accessible for practical implementation.

\subsection{Effectiveness Analysis of ESF + DS}
To further evaluate the performance of our proposed ESF and DS decoder in real-world scenarios, we choose the widely recognized TSPLIB \cite{Reinelt1991} and CVRPLIB \cite{UCHOA2017} as the benchmarking datasets for testing. Considering the training sizes of AMDKD ($n_{\textit{tr}}\in\{50, 100\}$) and OMNI-VRP ($n_{\textit{tr}}\in[50, 200]$), we choose TSP instances from TSPLIB with node sizes ranging from 52 to 200 and CVRP instances from CVRPLIB with node sizes ranging from 32 to 101 as our test cases. To test the consistence of our model's performance, we repeat the experiments on TSPLIB for 10 times, each time with a different random seed. The results are reported as the mean $\pm$ margin of error, where the margin of error reflecting the 95\% confidence interval.

\begin{table}[!t]
\small
\centering
\caption{Generalization performance of baseline models on CVRPLIB (sets\{A, B, E, F, M, P\})}\label{CVRPLIB}
\begin{threeparttable}
\begin{tabular}{l|l|cccc}
\toprule
\multirow{2}{*}{Size} & \multirow{2}{*}{\#} & \multicolumn{4}{c}{Optimality gap (\%) $\downarrow$}   \\
 & & OMNI-VRP & OMNI-VRP$^*$ & AMDKD & {DS} \\
 \midrule
32-48& 30&  4.20& 3.12& \textbf{1.40}& 1.78\\
50-69& 35&  3.23& 2.70& 1.79& \textbf{1.51}\\
 70-101& 18&  6.99& 6.29& 4.86& \textbf{4.77}\\
  \midrule
  All& 75&  4.12& 3.35& \textbf{2.04}& 2.05\\
 \bottomrule
\end{tabular}
\end{threeparttable}
\end{table}

\begin{figure*}[!t]
\centering
\subfigure[$n_{\textit{tr}}=50\rightarrow n_{\textit{te}}=100, \log_{n_{\textit{tr}}}{n_{\textit{te}}}=1.177$]{\includegraphics[width=0.9\columnwidth]{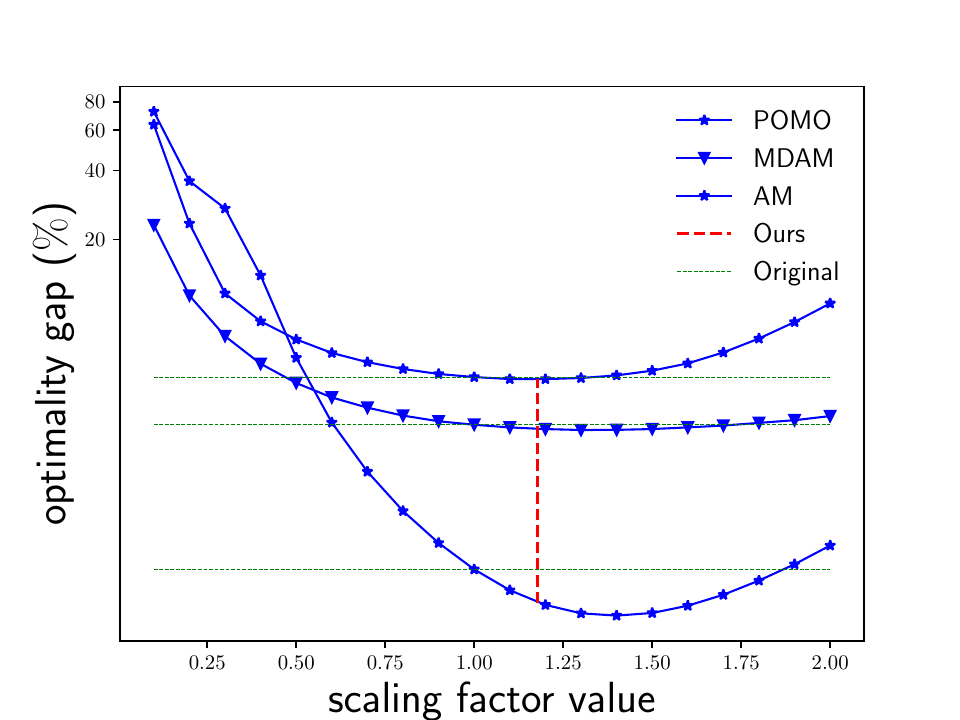}\label{fig5a}}
\subfigure[$n_{\textit{tr}}=100\rightarrow n_{\textit{te}}=50, \log_{n_{\textit{tr}}}{n_{\textit{te}}}=0.849$]{\includegraphics[width=0.9\columnwidth]{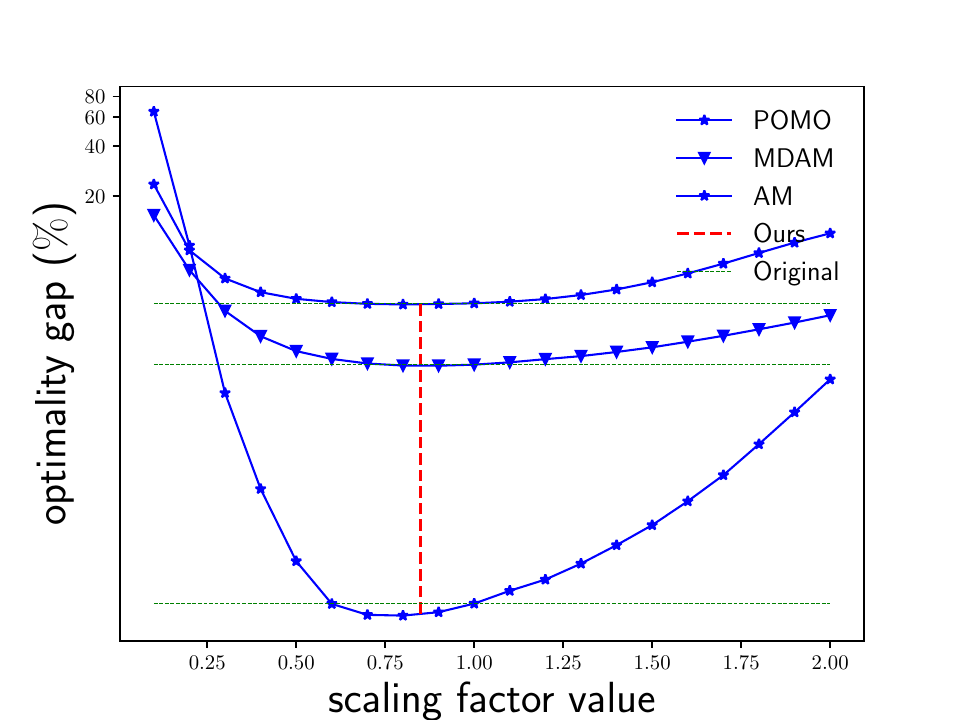}\label{fig5b}}
\subfigure[$n_{\textit{tr}}=50\rightarrow n_{\textit{te}}=200,  \log_{n_{\textit{tr}}}{n_{\textit{te}}}=1.354$]{\includegraphics[width=0.9\columnwidth]{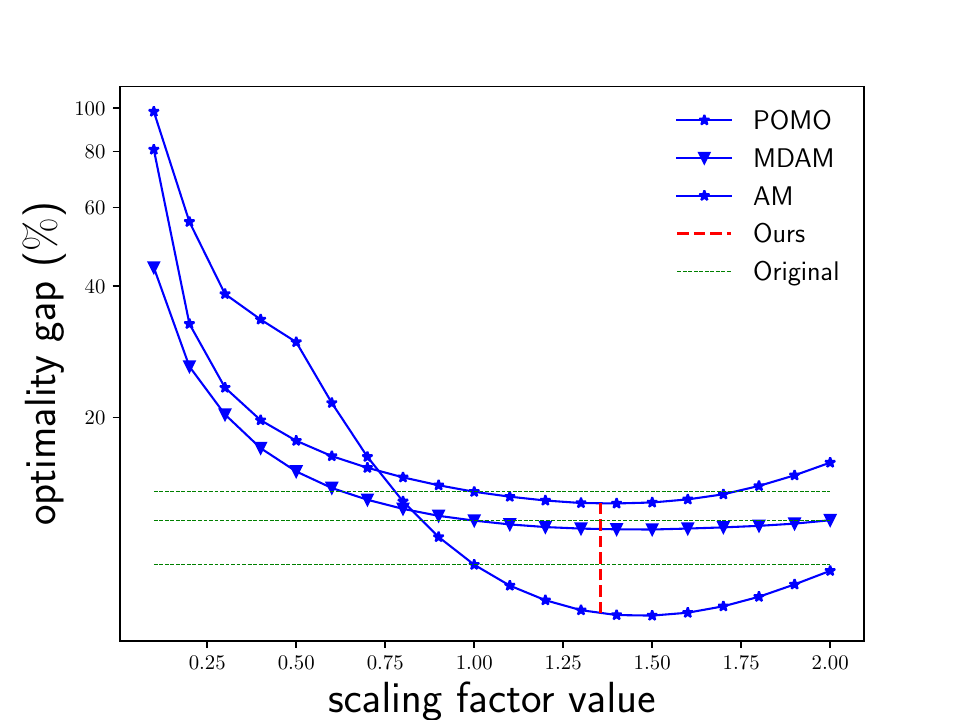}\label{fig5c}}
\subfigure[$n_{\textit{tr}}=100\rightarrow n_{\textit{te}}=200, \log_{n_{\textit{tr}}}{n_{\textit{te}}}=1.150$]{\includegraphics[width=0.9\columnwidth]{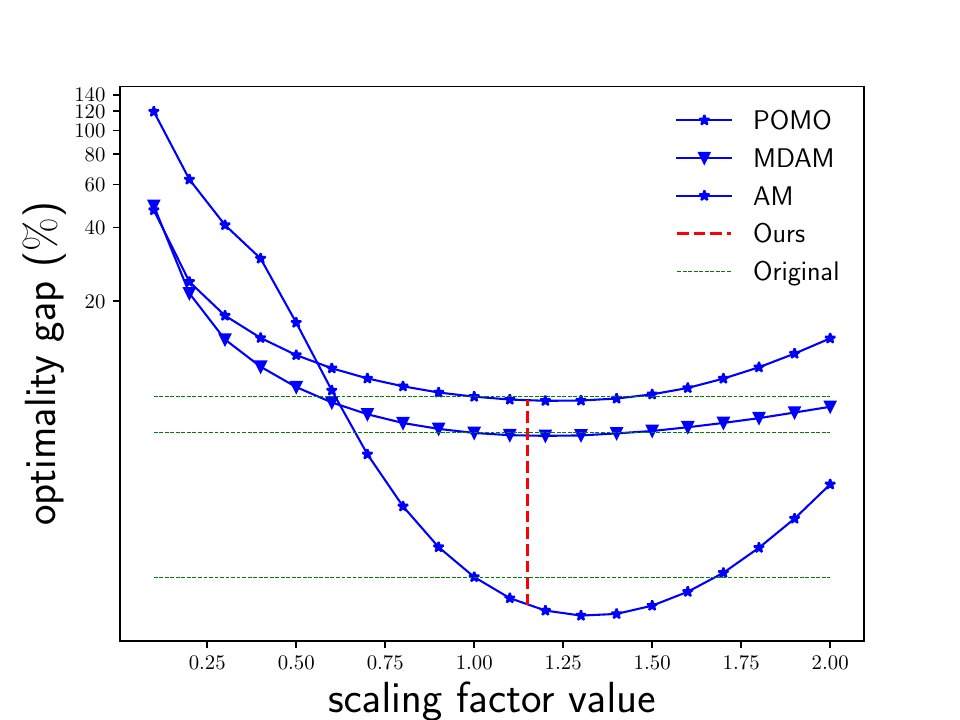}\label{fig5d}}
\caption{Generalization performance of baseline models with different scaling factors applied on TSPs.}\label{fig5}
\end{figure*}

\begin{figure*}[!t]
\centering
\subfigure[TSP-50]{\includegraphics[width=\columnwidth]{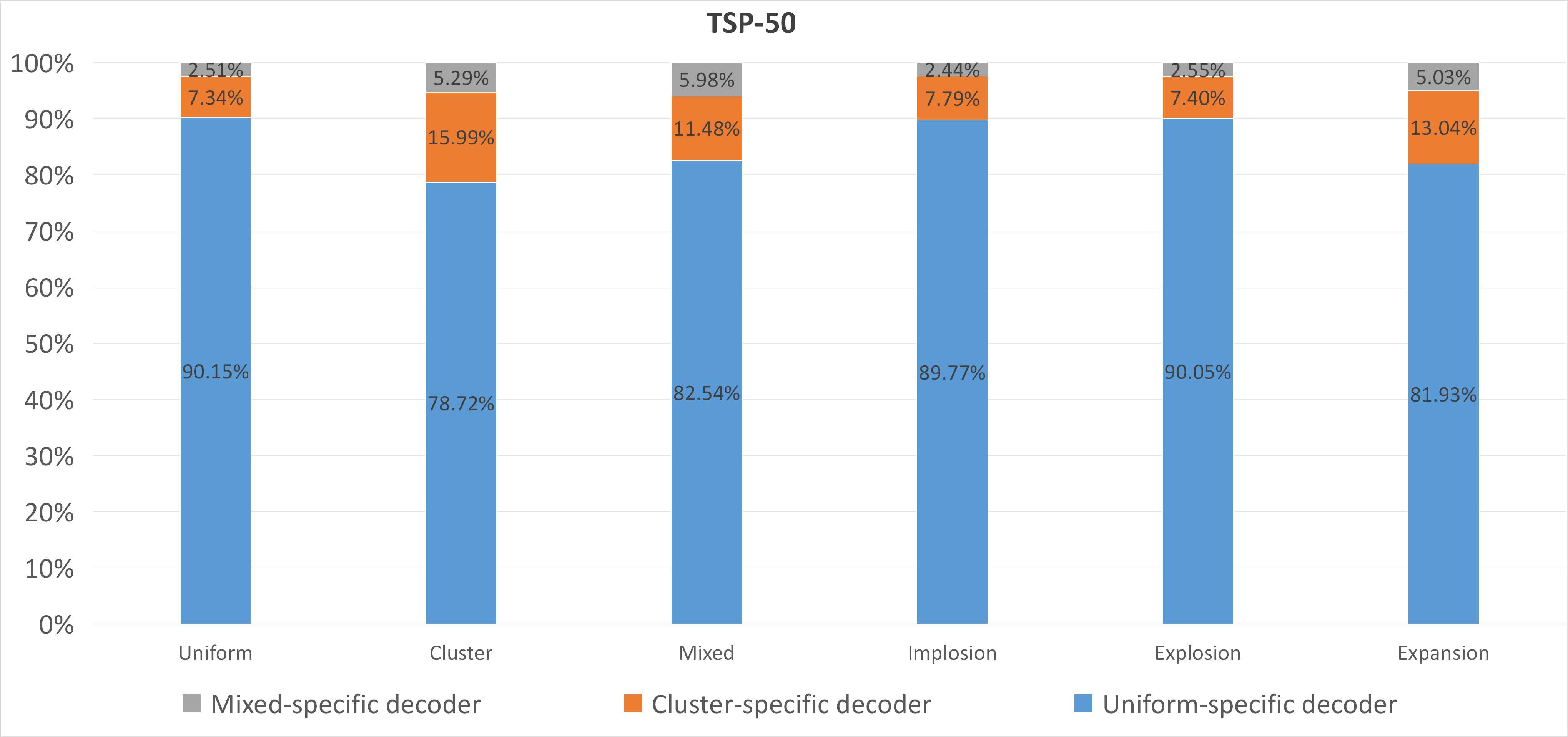}\label{fig6a}}
\subfigure[TSP-100]{\includegraphics[width=\columnwidth]{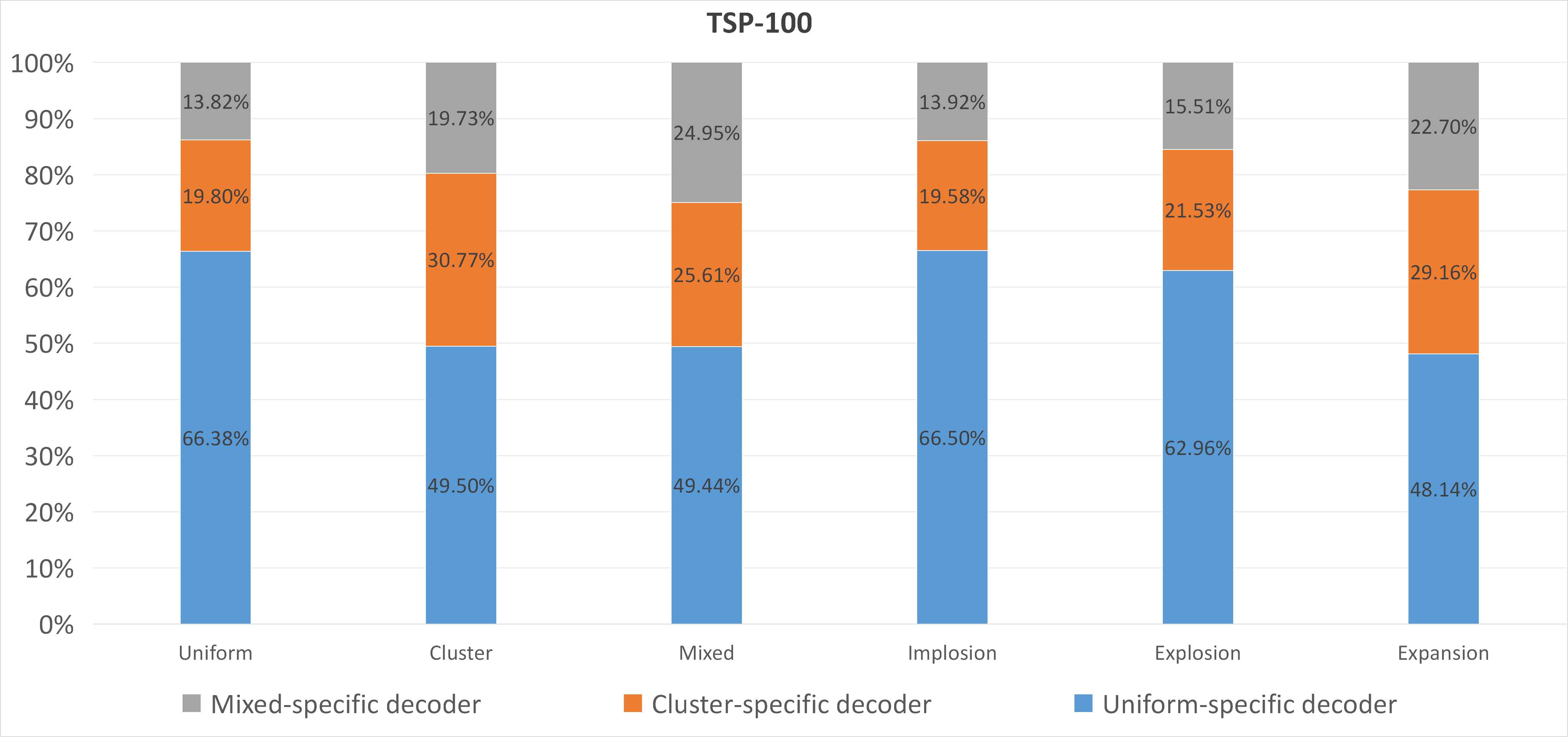}\label{fig6b}}
\subfigure[CVRP-50]{\includegraphics[width=\columnwidth]{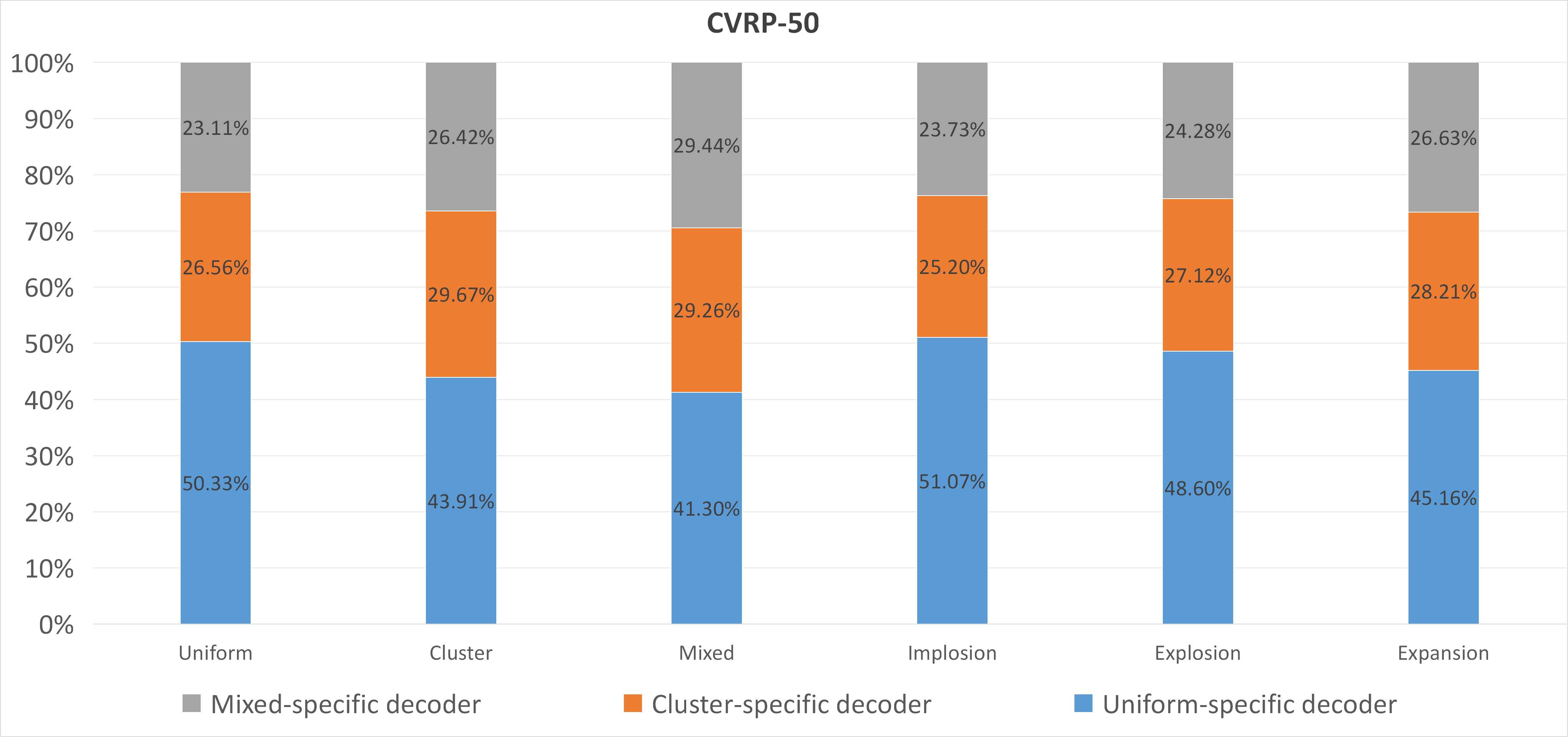}\label{fig6c}}
\subfigure[CVRP-100]{\includegraphics[width=\columnwidth]{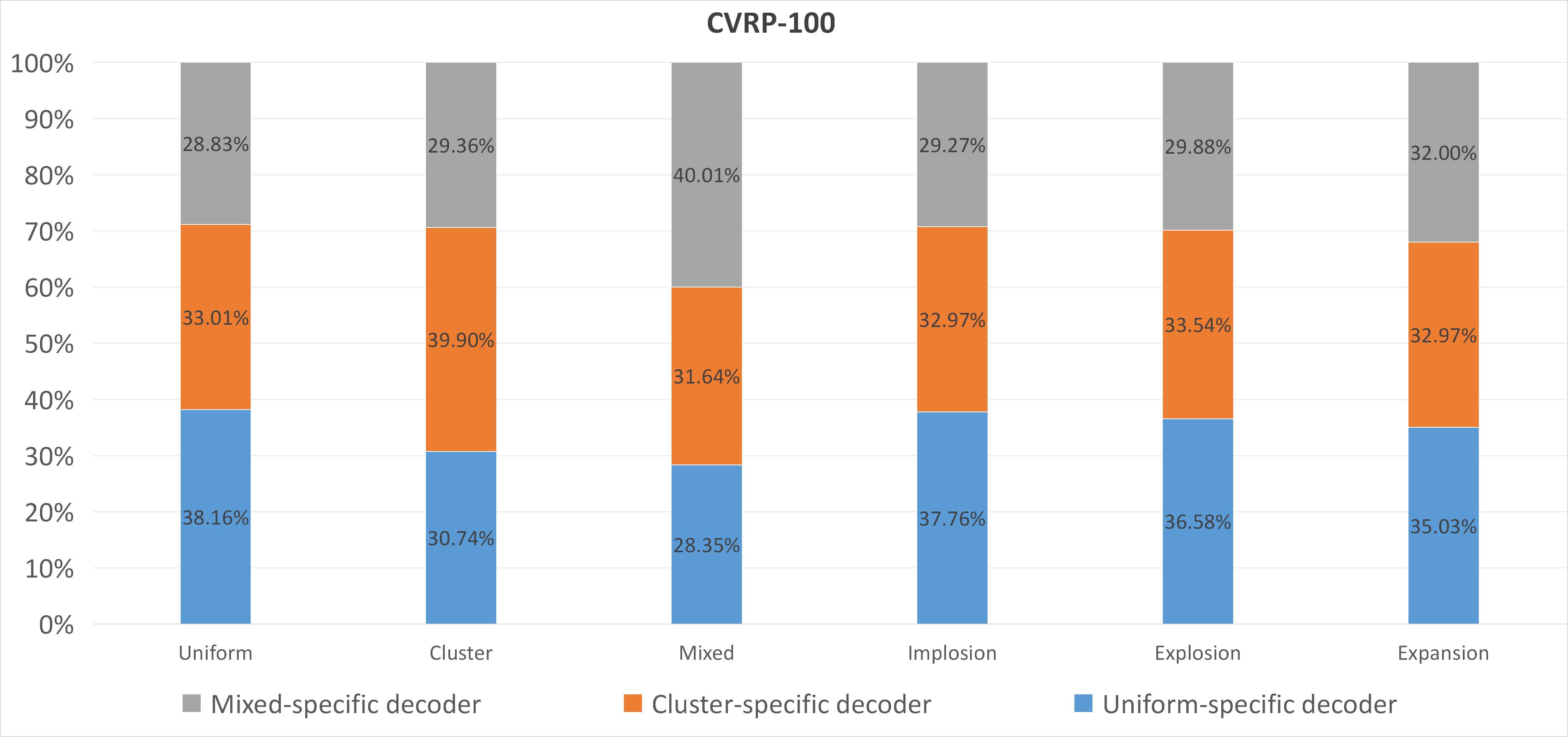}\label{fig6d}}

\caption{Proportion diagrams of the three distribution-specific decoders occupying the optimal choice on the test instances of the six distribution patterns.}\label{fig6}
\end{figure*}

The results presented in Tables~\ref{TSPLIB} (TSPLIB) and \ref{CVRPLIB} (CVRPLIB) indicate that our proposed DS decoder yields results comparable to the well-designed AMDKD and outperforms OMNI-VRP by simply imposing a model architecture improvement method. 
Moreover, as shown in Table~\ref{TSPLIB}, the low margin of error indicates that our model exhibits statistically robust performance. Remarkably, our method produces solutions identical to the optimal solutions provided by TSPLIB in the following three instances: ``pr76", ``pr124", and ``rd100". To the best of our knowledge, our approach represents the first neural construction model capable of achieving optimal solutions in TSPLIB without requiring iterative update processes. Furthermore, as shown in Tables~\ref{tab_ds} and \ref{TSPLIB}, the proposed ESF consistently improves the performance of baseline models and the DS decoder in all test instances with different distribution patterns and shows no preference toward any specific distribution patterns. These experimental results further illustrate the effectiveness of enhancing model generalization ability through our proposed components.

\subsection{Ablation Studies}
To assess the effectiveness of the proposed ESF and DS decoder, we further conduct two sets of experiments as follows.

Firstly, we conduct ablation studies on TSPs to assess the sensitivity of ESF. Specifically, we evaluate the size generalization performance of baseline models (AM, POMO, MDAM) using different scaling factors with values ranging from 0.1 to 2 during testing. We present the results in Figure~\ref{fig5}, where the model using the scaling factor with a value of $1$ represents the model's original performance (i.e., the line in cyan). The results suggest that using the scaling factor with values exceeding 1 and below 1 effectively reduce the optimality gap when scaling up (see Figures~\ref{fig5a}, \ref{fig5c} and \ref{fig5d}) and scaling down (see Figure~\ref{fig5b}) TSP instances, respectively, as shown in the area below the line in cyan for each model. This showcases the feasibility of using scaling factors to enhance the model's size generalization performance. Nonetheless, the effectiveness of the scaling factor is constrained within a specific range, beyond which the model's performance undergoes a considerable degradation. However, the result of our proposed ESF (i.e., the red line) consistently falls within the appropriate range, demonstrating the effectiveness of its value setting. It is worth noting that the proposed ESF's value may not always be the optimal value of the scaling factor. This is because a fixed scaling factor cannot account for variations introduced by the model using different inference methods, such as single rollout and multiple rollouts employed by AM and POMO, respectively. Model-specific scaling factors inevitably incorporate features unique to different models, resulting in a significant increase in computational cost. This contradicts with our pursuit of cost reduction in this work.

Secondly, we assess the efficacy of the DS decoder in learning from varying distribution patterns. Specifically, we employ the three DS decoders trained on the distribution patterns $\textit{d}_\textit{U}$, $\textit{d}_\textit{C}$, and $\textit{d}_\textit{M}$, respectively, to solve 10,000 VRPs for each of the six distribution patterns ($\textit{d}_\textit{U}$, $\textit{d}_\textit{C}$, $\textit{d}_\textit{M}$, $\textit{d}_\textit{I}$, $\textit{d}_\textit{Eo}$, and $\textit{d}_\textit{Ea}$). Then, we compute the percentage of the optimal (i.e., the best solution among the three DS decoders) choice made by each DS decoder for each distribution pattern. The results, as presented in Figure~\ref{fig6}, suggest that while the Uniform distribution-specific decoder dominates in most optimal choice cases, the other two DS decoders exhibit higher optimal choice percentages under the training distribution pattern compared to other distribution patterns. For instance, when solving TSP-50 instances, the Cluster distribution-specific decoder shows a optimal choice percentage of $15.99\%$ on the test instances of the Cluster distribution, while exhibiting optimal choice percentages of $\{7.34\%, 11.48\%, 7.79\%, 7.40\%, 13.04\%\}$ on the remaining five distribution patterns, respectively. This finding demonstrates that our DS decoder effectively learns distribution-dependent features. Furthermore, the dominance of the Uniform distribution-specific decoder underscores the significance of the Uniform distribution, explaining its extensive adoption in the training of almost all neural VRP solvers \cite{Kool2019}. Nevertheless, as the challenge in problem-solving increases\footnote{It is generally acknowledged that solving VRPs of larger sizes is more challenging than those of smaller sizes, and solving CVRPs is more challenging than TSPs at the same scale.}, the Uniform distribution hypothesis space faces growing challenges in covering other distribution patterns, as evidenced by the declining optimal choice percentage of the Uniform distribution-specific decoder (from Figures~\ref{fig6a} to \ref{fig6d}). This finding highlights the inherent constraints in relying solely on a single distribution, or even one formed from a mix of multiple distribution patterns, such as the Uniform-Cluster mixed distribution $\textit{d}_\textit{M}$. Consequently, employing our DS decoder to explicitly model multiple distribution patterns is anticipated to demonstrate efficacy in solving more intricate VRPs going forward.

Furthermore, we assess the robustness of the DS decoder when
trained on different distribution patterns. Specifically, we utilize Uniform-rectangle ($\textit{d}_\textit{Ur}$), Diagonal ($\textit{d}_\textit{D}$), and Gaussian ($\textit{d}_\textit{G}$) distribution patterns from \cite{Xin2022} for training and compare the performance with those using the DS decoder trained on Uniform ($\textit{d}_\textit{U}$), Cluster ($\textit{d}_\textit{C}$), and Uniform-Cluster mixed~($\textit{d}_\textit{M}$) distribution patterns. As shown in Table~\ref{ESF_Abla}, the performance differences among the models trained on these diverse distributions are minimal, demonstrating that the DS decoder maintains a high level of robustness when trained on diverse distribution patterns.

\begin{table}[!t]
\small
\centering
\caption{Results of the DS decoder trained on different distribution patterns}\label{ESF_Abla}
\begin{threeparttable}
\resizebox{1\columnwidth}{!}{
\begin{tabular}{ll|cccccc}
\toprule
\multicolumn{2}{l|}{\multirow{2}{*}{Model, $d_i$}} & \multicolumn{6}{c}{TSP-50, optimality gap (\%)}   \\
& &\multicolumn{1}{|c}{$\textit{d}_\textit{U}$} & $\textit{d}_\textit{C}$ & $\textit{d}_\textit{M}$ & $\textit{d}_\textit{I}$ & $\textit{d}_\textit{Eo}$ & \multicolumn{1}{c}{$\textit{d}_\textit{Ea}$}\\
 \midrule
 \multicolumn{2}{l|}{DS, $\textit{d}_\textit{U}, \textit{d}_\textit{C}, \textit{d}_\textit{M}$} & 0.02&  0.03& 0.04& 0.02 & 0.01 & 0.03 \\
 \multicolumn{1}{l}{\multirow{3}{*}{\tabincell{l}{DS with $n_{\mathcal{D}}=3$ \\ decoders}}} & min & 0.01 & 0.02 & 0.03 & 0.02 & 0.01 & 0.02 \\
 & max & 0.03 & 0.05 & 0.06 & 0.03 & 0.02 & 0.05\\

 & mean & 0.02 & 0.03 & 0.04 & 0.02 & 0.01 & 0.03\\
  \midrule
   \multicolumn{2}{l|}{DS, $\textit{d}_\textit{Ur}, \textit{d}_\textit{D}, \textit{d}_\textit{G}$} & 0.02&  0.08& 0.07& 0.02 & 0.01 & 0.09 \\
 \multicolumn{1}{l}{\multirow{3}{*}{\tabincell{l}{DS with $n_{\mathcal{D}}=3$ \\ decoders}}} & min & 0.01 & 0.05 & 0.05 & 0.02 & 0.01 & 0.05 \\
 & max & 0.03 & 0.10 & 0.10 & 0.03 & 0.02 & 0.11\\

 & mean & 0.02 & 0.07 & 0.07 & 0.02 & 0.02 & 0.08\\
 \bottomrule
\end{tabular}}
\end{threeparttable}
\end{table}

\section{Conclusion}
This paper explores the generalization of neural VRP solvers from the lens of model architecture, and proposes two generic components to enhance the size and distribution generalization, respectively. Extensive experimental results demonstrate the effectiveness of the proposed components and demonstrate the feasibility of enhancing generalization through lightweight model architecture improvement methods. We believe that our study offers a novel perspective within the neural CO community.

While both components are generic, the observed improvement is limited in few experimental settings. Going forward, we plan to 1) develop a learnable scaling factor to obtain further performance gain; 2) develop an additional NN capable of outputting VRPs with notably varying distribution patterns, then integrate it with our DS decoder for covering a broader distribution hypothesis space; and 3) synthesize these two methods to achieve an all-encompassing generalization for practical VRPs.

\bibliographystyle{elsarticle-num}
\bibliography{MyBib1}

\begin{thebibliography}{10}
\expandafter\ifx\csname url\endcsname\relax
  \def\url#1{\texttt{#1}}\fi
\expandafter\ifx\csname urlprefix\endcsname\relax\def\urlprefix{URL }\fi
\expandafter\ifx\csname href\endcsname\relax
  \def\href#1#2{#2} \def\path#1{#1}\fi

\bibitem{Kim2015}
G.~Kim, Y.-S. Ong, C.~K. Heng, P.~S. Tan, N.~A. Zhang, City vehicle routing problem ({C}ity {VRP}): A review, IEEE Transactions on Intelligent Transportation Systems 16 (2015) 1654--1666.

\bibitem{Applegate2007}
D.~L. Applegate, R.~E. Bixby, V.~Chvátal, W.~J. Cook, The Traveling Salesman Problem: A Computational Study, Princeton University Press, 2007.

\bibitem{Helsgaun2017}
K.~Helsgaun, An extension of the {L}in--{K}ernighan--{H}elsgaun {TSP} solver for constrained traveling salesman and vehicle routing problems, Roskilde: Roskilde University (2017) 24--50.

\bibitem{Zhou2023}
J.~Zhou, Y.~Wu, W.~Song, Z.~Cao, J.~Zhang, Towards omni-generalizable neural methods for vehicle routing problems, in: Proceedings of International Conference on Machine Learning, Vol. 202, 2023, pp. 42769--42789.

\bibitem{Li2021}
K.~Li, T.~Zhang, R.~Wang, W.~Qin, H.~He, H.~Huang, Research reviews of combinatorial optimization methods based on deep reinforcement learning, Zidonghua Xuebao/Acta Automatica Sinica 47 (2021) 2521--2537.

\bibitem{Kool2019}
W.~Kool, H.~van Hoof, M.~Welling, Attention, learn to solve routing problems!, in: Proceedings of International Conference on Learning Representations, 2019.

\bibitem{Kwon2020}
Y.-D. Kwon, J.~Choo, B.~Kim, I.~Yoon, Y.~Gwon, S.~Min, {POMO}: Policy optimization with multiple optima for reinforcement learning, in: Proceedings of Advances in Neural Information Processing Systems, 2020, pp. 21188--21198.

\bibitem{Bi2022}
J.~Bi, Y.~Ma, J.~Wang, Z.~Cao, J.~Chen, Y.~Sun, Y.~M. Chee, Learning generalizable models for vehicle routing problems via knowledge distillation, in: Proceedings of Advances in Neural Information Processing Systems, 2022, pp. 31226--31238.

\bibitem{Sun2023}
Z.~Sun, Y.~Yang, {DIFUSCO}: Graph-based diffusion solvers for combinatorial optimization, in: Proceedings of Advances in Neural Information Processing Systems, Vol.~36, 2023, pp. 3706--3731.

\bibitem{Pan2023}
X.~Pan, Y.~Jin, Y.~Ding, M.~Feng, L.~Zhao, L.~Song, J.~Bian, H-{TSP}: Hierarchically solving the large-scale traveling salesman problem, in: Proceedings of the AAAI Conference on Artificial Intelligence, Vol.~37, 2023, pp. 9345--9353.

\bibitem{Reinelt1991}
G.~Reinelt, {TSPLIB-A} traveling salesman problem library, ORSA Journal on Computing 3(4) (1991) 376--384.

\bibitem{UCHOA2017}
E.~Uchoa, D.~Pecin, A.~Pessoa, M.~Poggi, T.~Vidal, A.~Subramanian, New benchmark instances for the capacitated vehicle routing problem, European Journal of Operational Research 257~(3) (2017) 845--858.

\bibitem{Xin2021}
L.~Xin, W.~Song, Z.~Cao, J.~Zhang, Multi-decoder attention model with embedding glimpse for solving vehicle routing problems, in: Proceeding of the AAAI Conference on Artificial Intelligence, 2021, pp. 12042--12049.

\bibitem{Luo2023}
F.~Luo, X.~Lin, F.~Liu, Q.~Zhang, Z.~Wang, Neural combinatorial optimization with heavy decoder: Toward large scale generalization, in: Proceedings of Advances in Neural Information Processing Systems, Vol.~36, 2023, pp. 8845--8864.

\bibitem{Drakulic2023}
D.~Drakulic, S.~Michel, F.~Mai, A.~Sors, J.-M. Andreoli, {BQ-NCO}: Bisimulation quotienting for efficient neural combinatorial optimization, in: Proceedings of Advances in Neural Information Processing Systems, 2023, pp. 77416--77429.

\bibitem{Zhu2023}
T.~Zhu, X.~Shi, X.~Xu, J.~Cao, An accelerated end-to-end method for solving routing problems, Neural Networks 164 (2023) 535--545.

\bibitem{YANG2023}
H.~Yang, M.~Zhao, L.~Yuan, Y.~Yu, Z.~Li, M.~Gu, Memory-efficient transformer-based network model for traveling salesman problem, Neural Networks 161 (2023) 589--597.

\bibitem{Xin2021a}
L.~Xin, W.~Song, Z.~Cao, J.~Zhang, Neurolkh: Combining deep learning model with {L}in-{K}ernighan-{H}elsgaun heuristic for solving the traveling salesman problem, in: Proceedings of Advances in Neural Information Processing Systems, Vol.~34, 2021, pp. 7472--7483.

\bibitem{Xiao2020}
Y.~Xiao, Z.~Xiao, X.~Feng, Z.~Chen, L.~Kuang, L.~Wang, A novel computational model for predicting potential lnc{RNA}-disease associations based on both direct and indirect features of lnc{RNA}-disease pairs, BMC Bioinformatics 21 (2020) 1--22.

\bibitem{Wu2024}
X.~Wu, D.~Wang, L.~Wen, Y.~Xiao, C.~Wu, Y.~Wu, C.~Yu, D.~L. Maskell, Y.~Zhou, Neural combinatorial optimization algorithms for solving vehicle routing problems: A comprehensive survey with perspectives, arXiv:2406.00415 (2024).
\newblock \href {http://arxiv.org/abs/2406.00415} {\path{arXiv:2406.00415}}.

\bibitem{Zhao2024}
S.~Zhao, S.~Gu, A deep reinforcement learning algorithm framework for solving multi-objective traveling salesman problem based on feature transformation, Neural Networks 176 (2024) 106359.

\bibitem{Zhou2024}
C.~Zhou, X.~Lin, Z.~Wang, X.~Tong, M.~Yuan, Q.~Zhang, Instance-conditioned adaptation for large-scale generalization of neural combinatorial optimization, arXiv:2405.01906 (2024).
\newblock \href {http://arxiv.org/abs/2405.01906} {\path{arXiv:2405.01906}}.

\bibitem{Fang2024}
H.~Fang, Z.~Song, P.~Weng, Y.~Ban, {INV}i{T}: A generalizable routing problem solver with invariant nested view transformer, in: Proceedings of International Conference on Machine Learning, Vol. 235, 2024, pp. 12973--12992.

\bibitem{Vinyals2015}
O.~Vinyals, M.~Fortunato, N.~Jaitly, Pointer networks, in: Proceedings of Advances in Neural Information Processing Systems, 2015, pp. 2692--–2700.

\bibitem{Jin2023}
Y.~Jin, Y.~Ding, X.~Pan, K.~He, L.~Zhao, T.~Qin, L.~Song, J.~Bian, Pointerformer: {Deep} reinforced multi-pointer transformer for the traveling salesman problem, in: Proceedings of the {AAAI} {Conference} on {Artificial} {Intelligence}, Vol.~37, 2023, pp. 8132--8140.

\bibitem{Xiao2023}
Y.~Xiao, D.~Wang, B.~Li, H.~Chen, W.~Pang, X.~Wu, H.~Li, D.~Xu, Y.~Liang, Y.~Zhou, Reinforcement learning-based non-autoregressive solver for traveling salesman problems, IEEE Transactions on Neural Networks and Learning Systems (2024).

\bibitem{Min2023}
Y.~Min, Y.~Bai, C.~P. Gomes, Unsupervised learning for solving the travelling salesman problem, in: Proceedings of Advances in Neural Information Processing Systems, Vol.~36, 2023, pp. 47264--47278.

\bibitem{Kim2022}
M.~Kim, J.~Park, J.~Park, Sym-{NCO}: Leveraging symmetricity for neural combinatorial optimization, in: Proceedings of Advances in Neural Information Processing Systems, Vol.~35, 2022, pp. 1936--1949.

\bibitem{Kwon2021}
Y.-D. Kwon, J.~Choo, I.~Yoon, M.~Park, D.~Park, Y.~Gwon, Matrix encoding networks for neural combinatorial optimization, in: Proceedings of Advances in Neural Information Processing Systems, Vol.~34, 2021, pp. 5138--5149.

\bibitem{Felix2023}
F.~Chalumeau, S.~Surana, C.~Bonnet, N.~Grinsztajn, A.~Pretorius, A.~Laterre, T.~Barrett, Combinatorial optimization with policy adaptation using latent space search, in: Proceedings of Advances in Neural Information Processing Systems, Vol.~36, 2023, pp. 7947--7959.

\bibitem{Hottung2022}
A.~Hottung, Y.-D. Kwon, K.~Tierney, Efficient active search for combinatorial optimization problems, in: Proceedings of International Conference on Learning Representations, 2022.

\bibitem{Choo2022}
J.~Choo, Y.-D. Kwon, J.~Kim, J.~Jae, A.~Hottung, K.~Tierney, Y.~Gwon, Simulation-guided beam search for neural combinatorial optimization, in: Proceedings of Advances in Neural Information Processing Systems, Vol.~35, 2022, pp. 8760--8772.

\bibitem{Hottung2024}
A.~Hottung, M.~Mahajan, K.~Tierney, Polynet: Learning diverse solution strategies for neural combinatorial optimization, arXiv:2402.14048 (2024).
\newblock \href {http://arxiv.org/abs/2402.14048} {\path{arXiv:2402.14048}}.

\bibitem{Garmendia2024}
A.~I.~Garmendia, Q.~Cappart, J.~Ceberio, A.~Mendiburu, {MARCO}: A memory-augmented reinforcement framework for combinatorial optimization, in: Proceedings of International Joint Conference on Artificial Intelligence, 2024, p. 6931–6939.

\bibitem{Xiao2024}
Y.~Xiao, D.~Wang, B.~Li, M.~Wang, X.~Wu, C.~Zhou, Y.~Zhou, Distilling autoregressive models to obtain high-performance non-autoregressive solvers for vehicle routing problems with faster inference speed, in: Proceedings of the AAAI Conference on Artificial Intelligence, Vol.~38, 2024, pp. 20274--20283.

\bibitem{Hudson2022}
B.~Hudson, Q.~Li, M.~Malencia, A.~Prorok, Graph neural network guided local search for the traveling salesperson problem, in: Proceedings of International Conference on Learning Representations, 2022.

\bibitem{Ma2023}
Y.~Ma, Z.~Cao, Y.~M. Chee, Learning to search feasible and infeasible regions of routing problems with flexible neural k-opt, in: Proceedings of Advances in Neural Information Processing Systems, Vol.~36, 2023, pp. 49555--49578.

\bibitem{Ye2023}
H.~Ye, J.~Wang, Z.~Cao, H.~Liang, Y.~Li, Deep{ACO}: Neural-enhanced ant systems for combinatorial optimization, in: Proceedings of Advances in Neural Information Processing Systems, Vol.~36, 2023, pp. 43706--43728.

\bibitem{Kim2024}
M.~Kim, S.~Choi, H.~Kim, J.~Son, J.~Park, Y.~Bengio, Ant colony sampling with gflownets for combinatorial optimization, arXiv:2403.07041 (2024).
\newblock \href {http://arxiv.org/abs/2403.07041} {\path{arXiv:2403.07041}}.

\bibitem{Bengio2021}
Y.~Bengio, A.~Lodi, A.~Prouvost, Machine learning for combinatorial optimization: A methodological tour d'horizon, European Journal of Operational Research 290 (2021) 405--421.

\bibitem{Joshi2022}
C.~K. Joshi, Q.~Cappart, L.-M. Rousseau, T.~Laurent, Learning the travelling salesperson problem requires rethinking generalization, Constraints 27 (2022) 70--98.

\bibitem{Ahmad2023}
A.~Bdeir, J.~K. Falkner, L.~Schmidt-Thieme, Attention, filling in the gaps for generalization in routing problems, in: Proceedings of Machine Learning and Knowledge Discovery in Databases, 2023, pp. 505--520.

\bibitem{Wang2024}
Y.~Wang, Y.-H. Jia, W.-N. Chen, Y.~Mei, Distance-aware attention reshaping: Enhance generalization of neural solver for large-scale vehicle routing problems, arXiv:2401.06979 (2024).
\newblock \href {http://arxiv.org/abs/2401.06979} {\path{arXiv:2401.06979}}.

\bibitem{Fu2021}
Z.~Fu, K.~Qiu, H.~Zha, Generalize a small pre-trained model to arbitrarily large {TSP} instances, in: Proceedings of the AAAI Conference on Artificial Intelligence, 2021, pp. 7474--7482.

\bibitem{Li2021a}
S.~Li, Z.~Yan, C.~Wu, Learning to delegate for large-scale vehicle routing, in: Proceedings of Advances in Neural Information Processing Systems, Vol.~34, 2021, pp. 26198--26211.

\bibitem{Hou2023}
Q.~Hou, J.~Yang, Y.~Su, X.~Wang, Y.~Deng, Generalize learned heuristics to solve large-scale vehicle routing problems in real-time, in: Proceedings of International Conference on Learning Representations, 2023.

\bibitem{Ye2024}
H.~Ye, J.~Wang, H.~Liang, Z.~Cao, Y.~Li, F.~Li, {GLOP}: Learning global partition and local construction for solving large-scale routing problems in real-time, in: Proceedings of the AAAI Conference on Artificial Intelligence, Vol.~38, 2024, pp. 20284--20292.

\bibitem{Yu2024}
K.~Yu, H.~Zhao, Y.~Huang, R.~Yi, K.~Xu, C.~Zhu, {DISCO}: Efficient diffusion solver for large-scale combinatorial optimization problems, arXiv:2406.19705 (2024).
\newblock \href {http://arxiv.org/abs/2406.19705} {\path{arXiv:2406.19705}}.

\bibitem{Zhang2022}
Z.~Zhang, Z.~Zhang, X.~Wang, W.~Zhu, Learning to solve travelling salesman problem with hardness-adaptive curriculum, in: Proceedings of the {AAAI} {Conference} on {Artificial} {Intelligence}, 2022, pp. 9136--9144.

\bibitem{Geisler2022}
S.~Geisler, J.~Sommer, J.~Schuchardt, A.~Bojchevski, S.~G{\"u}nnemann, Generalization of neural combinatorial solvers through the lens of adversarial robustness, in: Proceedings of International Conference on Learning Representations, 2022.

\bibitem{Zhai2023}
Z.-M. Zhai, M.~Moradi, L.-W. Kong, Y.-C. Lai, Detecting weak physical signal from noise: A machine-learning approach with applications to magnetic-anomaly-guided navigation, Physical Review Applied 19 (2023) 034030.

\bibitem{Moradi2024}
M.~Moradi, S.~Panahi, Z.-M. Zhai, Y.~Weng, J.~Dirkman, Y.-C. Lai, {Heterogeneous reinforcement learning for defending power grids against attacks}, APL Machine Learning 2 (2024) 026121.

\bibitem{Zhai2023a}
Z.-M. Zhai, M.~Moradi, L.-W. Kong, B.~Glaz, H.~Mulugeta, Y.-C. Lai, Model-free tracking control of complex dynamical trajectories with machine learning, Nature Communications 14 (2023) 5698.

\bibitem{Williams1992}
R.~J. Williams, Simple statistical gradient-following algorithms for connectionist reinforcement learning, Machine Learning 8 (1992) 229--256.

\bibitem{Zhong2021a}
X.~Zhong, Z.~Wu, T.~Tan, G.~Lin, Q.~Wu, Mv-ton: Memory-based video virtual try-on network, in: Proceedings of the ACM International Conference on Multimedia, 2021, pp. 908--916.

\bibitem{Chen2023}
Y.~Chen, K.~Zhou, Y.~Bian, B.~Xie, B.~Wu, Y.~Zhang, M.~KAILI, H.~Yang, P.~Zhao, B.~Han, J.~Cheng, Pareto invariant risk minimization: Towards mitigating the optimization dilemma in out-of-distribution generalization, in: Proceedings of International Conference on Learning Representations, 2023.

\bibitem{haug2021}
J.~Haug, G.~Kasneci, Learning parameter distributions to detect concept drift in data streams, in: Proceedings of the International Conference on Pattern Recognition, 2021, pp. 9452--9459.

\bibitem{Xin2022}
L.~Xin, W.~Song, Z.~Cao, J.~Zhang, Generative adversarial training for neural combinatorial optimization models, https://openreview.net/forum?id=9vsRT9mc7U (2022).

\end{thebibliography}
\end{document}